%% file: instance-value-function-revised.tex
\documentclass[11pt,twoside]{article}

\usepackage{fullpage}
\usepackage{multirow}
\input{commands-cleaned.txt}

\input{DinosaurMacros}

\newcommand{\rstar}{\reward}

\newcommand{\sigmarsqindex}{\ensuremath{\rho_j^2(\rstar)}}

\newcommand{\MspaceVar}{\ensuremath{\Mspace_{\mathsf{var}}}}
\newcommand{\MspaceValFun}{\ensuremath{\Mspace_{\mathsf{vfun}}}}
\newcommand{\MspaceRew}{\ensuremath{\Mspace_{\mathsf{rew}}}}

\newcommand{\spannorm}[1]{\ensuremath{\|#1\|}_{\myspan}}

\newcommand{\MRP}{\ensuremath{\mathcal{R}}}
\newcommand{\betaplug}{\ensuremath{\widehat{\beta}_{\mbox{\tiny{plug}}}}}
\newcommand{\betamom}{\ensuremath{\widehat{\beta}_{\mbox{\tiny{MoM}}}}}

\makeatletter
\long\def\@makecaption#1#2{
        \vskip 0.8ex
        \setbox\@tempboxa\hbox{\small {\bf #1:} #2}
        \parindent 1.5em  
        \dimen0=\hsize
        \advance\dimen0 by -3em
        \ifdim \wd\@tempboxa >\dimen0
                \hbox to \hsize{
                        \parindent 0em
                        \hfil 
                        \parbox{\dimen0}{\def\baselinestretch{0.96}\small
                                {\bf #1.} #2
                                } 
                        \hfil}
        \else \hbox to \hsize{\hfil \box\@tempboxa \hfil}
        \fi
        }
\makeatother

\begin{document}

\begin{center}

  {\bf{\LARGE{Instance-dependent $\ell_\infty$-bounds for policy evaluation in tabular reinforcement learning}}}

\vspace*{.2in}

{\large{
\begin{tabular}{ccc}
Ashwin Pananjady$^\star$ and Martin J. Wainwright$^{\dagger, \ddagger}$
\end{tabular}
}}
\vspace*{.2in}

\begin{tabular}{c}
$^\star$Simons Institute for the Theory of Computing, UC Berkeley \\
$^\dagger$Departments of EECS and Statistics, UC Berkeley \\
  $^\ddagger$Voleon Group, Berkeley
\end{tabular}

\vspace*{.2in}

\today

\vspace*{.2in}

\begin{abstract}
Markov reward processes (MRPs) are used to model stochastic phenomena
arising in operations research, control engineering, robotics, and
artificial intelligence, as well as communication and transportation
networks. In many of these cases, such as in the policy evaluation
problem encountered in reinforcement learning, the goal is to estimate
the long-term value function of such a process without access to the
underlying population transition and reward functions. Working with
samples generated under the synchronous model, we study the problem of
estimating the value function of an infinite-horizon, discounted 
MRP on finitely many states in the $\ell_\infty$-norm. 
We analyze both the standard plug-in
approach to this problem and a more robust variant, and establish
non-asymptotic bounds that depend on the (unknown) problem instance,
as well as data-dependent bounds that can be evaluated based on the
observations of state-transitions and rewards. 
We show that these approaches are minimax-optimal up to
constant factors over natural sub-classes of MRPs. Our analysis makes
use of a leave-one-out decoupling argument tailored to the policy
evaluation problem, one which may be of independent interest.
\end{abstract}
\end{center}

\section{Introduction} \label{sec:intro}

A variety of applications spanning science and engineering use Markov
reward processes as models for real-world phenomena, including
queuing systems, transportation networks, robotic exploration, game
playing, and epidemiology.  In some of these settings, the underlying
parameters that govern the process are known to the modeler, but in
others, these must be estimated from observed data. A salient example
of the latter setting, which forms the main motivation for this paper,
is the policy evaluation problem encountered in Markov decision
processes (MDPs) and reinforcement
learning~\cite{Bertsekas_dyn1,Bertsekas_dyn2,SutBar18}. Here an agent
operates in an environment whose dynamics are unknown: at each step,
it observes the current state of the environment, and takes an action
that changes its state according to some stochastic transition
function determined by the environment.  The goal is to evaluate the
utility of some policy---that is, a mapping from states to actions,
where utility is measured using rewards that the agent receives from
the environment.  These rewards are usually assumed to be additive
over time, and since the policy determines the action to be taken at
each state, the reward obtained at any time is simply a function of
the current state of the agent. Thus, this setting induces a Markov
reward process (MRP) on the state space, in which both the underlying
transitions and rewards are unknown to the agent.  The agent only
observes samples of state transitions and rewards.

Given these samples, the goal of the agent is to estimate the
\emph{value function} of the MRP.  As noted above, in the context of
Markov decision processes (MDPs), this problem is known as policy
evaluation.  The value function evaluated at a given state measures
the expected long-term reward accumulated by starting at that state
and running the underlying Markov chain.  In applications, this value
function encodes crucial information about the MRP.  For example,
there are MRPs in which the value function corresponds to the
probability of a power grid failing~\cite{frank2008reinforcement}, the
taxi times of flights in an airport~\cite{balakrishna2008estimating},
or the value of a board configuration in a game of
Go~\cite{silver2007reinforcement}.  Moreover, policy evaluation is an
important component of many policy optimization algorithms for
reinforcement learning, which use it as a sub-routine while searching
for good policies to deploy in the environment.

The focus of this paper is on understanding the policy evaluation
problem in finite-state (or tabular) MRPs in an instance-dependent manner, focusing
on the the generative setting in which the agent has access to a
simulator that generates samples from the underlying MRP.  In
particular, we would like guarantees on the \emph{sample complexity}
of policy evaluation---defined as the number of samples required to
obtain a value function estimate of some pre-specified error
tolerance---as a function of the agent's environment, i.e., the
transition and reward functions induced by the policy being evaluated.
Local guarantees of this form provide more guidance for algorithm
design in finite sample settings than their worst-case counterparts.
Indeed, this viewpoint underpins the important sub-field of local
minimax complexity studied widely in the statistics and optimization
literatures (e.g.,~\cite{cai2004adaptation,zhu2016local, WeiWai17}),
as well as in more recent work on online reinforcement learning
algorithms~\cite{zanette2019tighter}.

As a natural first step towards providing local guarantees for the
policy evaluation problem, we analyze the plug-in estimator for the
problem, which estimates the underlying transition and reward
functions from the samples, and outputs the value function of the MRP
in which these estimates correspond to the ground truth parameters. 
We also analyze a robust variant
of this approach, and provide minimax lower bounds that hold over
subsets of the parameter space.


\begin{table}[]
\resizebox{\textwidth}{!}{%
\begin{tabular}{|c|c|c|c|c|c|c|}
\hline
\begin{tabular}[c]{@{}c@{}} \text{} \\ \textbf{Problem} \\ \text{ } \end{tabular} &
  \textbf{Algorithm} &
  \textbf{Paper} &
  \textbf{Model} &
  \textbf{Sample-size} &
  \textbf{Guarantee} &
  \textbf{Technique} \\ \hline
\multirow{5}{*}{\begin{tabular}[c]{@{}c@{}}State-action \\ value \\ estimation \\ in MDPs\end{tabular}} &
  \multirow{2}{*}{Plug-in} &
  \cite{KeaSin99}, \cite{Kak03} &
  Synchronous &
  Non-asymptotic &
  Global, $\ell_\infty$ &
  Hoeffding \\ \cline{3-7} 
 &
   &
  \cite{AzaMunKap13} &
  Synchronous &
  Non-asymptotic &
  Global, $\ell_\infty$ &
  Bernstein \\ \cline{2-7} 
 &
  \multirow{3}{*}{\begin{tabular}[c]{@{}c@{}}Stochastic\\ approximation:\\ $Q$-learning \&\\ variants\end{tabular}} &
 \cite{borkar2000ode} &
  Synchronous &
  Asymptotic &
  \begin{tabular}[c]{@{}c@{}}Global, \\ conv. in dist.\end{tabular} &
  ODE method \\ \cline{3-7} 
 & &
  \cite{devraj2017fastest} &
  Synchronous &
  Asymptotic &
  \begin{tabular}[c]{@{}c@{}}Local, \\ conv. in dist.\end{tabular} &
  \begin{tabular}[c]{@{}c@{}}Asymptotic \\ normality \end{tabular} \\ \cline{3-7} 
 &
   &
   \begin{tabular}[c]{@{}c@{}} \cite{wainwright2019stochastic}, \\ \cite{chen2020finite} \end{tabular}
&
  Synchronous &
  Non-asymptotic &
  Local, $\ell_\infty$ &
  \begin{tabular}[c]{@{}c@{}}Bernstein, \\ Moreau envelope\end{tabular} \\ \cline{3-7} 
 &
   &
  \begin{tabular}[c]{@{}c@{}}\cite{Aza11}, \\ \cite{Sid18a}, \\ \cite{wainwright2019variance}\end{tabular} &
  Synchronous &
  Non-asymptotic &
  Global, $\ell_\infty$ &
  \begin{tabular}[c]{@{}c@{}}Bernstein, \\ variance reduction\end{tabular} \\ \hline
\multirow{3}{*}{\begin{tabular}[c]{@{}c@{}}Optimal \\ value \\ estimation \\ in MDPs\end{tabular}} &
  \multirow{2}{*}{Plug-in} &
  \cite{AzaMunKap13} &
  Synchronous &
  Non-asymptotic &
  Global, $\ell_\infty$ &
  Bernstein \\ \cline{3-7} 
 &
   &
  \cite{agarwal2019optimality} &
  Synchronous &
  Non-asymptotic &
  Global, $\ell_\infty$ &
  \begin{tabular}[c]{@{}c@{}}Bernstein + \\ decoupling\end{tabular} \\ \cline{2-7} 
 &
  \begin{tabular}[c]{@{}c@{}}Stochastic\\ approximation\end{tabular} &
  \cite{Sid18b} &
  Synchronous &
  Non-asymptotic &
  Global, $\ell_\infty$ &
  \begin{tabular}[c]{@{}c@{}}Bernstein + \\ variance reduction\end{tabular} \\ \hline
\multirow{6}{*}{\begin{tabular}[c]{@{}c@{}}Policy \\ evaluation \\ in MRPs\end{tabular}} &
  \textbf{\red{Plug-in}} &
  \textbf{\red{Current paper}} &
  \textbf{\red{Synchronous}} &
  \textbf{\red{Non-asymptotic}} &
  \textbf{\red{Local, $\ell_\infty$}} &
  \textbf{\red{\begin{tabular}[c]{@{}c@{}}Bernstein + \\ leave-one-out\end{tabular}}} \\ \cline{2-7} 
 &
  \multirow{2}{*}{\begin{tabular}[c]{@{}c@{}}Stochastic \\ approximation:\\ TD-learning\end{tabular}} &
  \begin{tabular}[c]{@{}c@{}}\cite{tadic2004almost}, \cite{polyak1992acceleration}, \\ \cite{jaakkola1994convergence}, \cite{borkar2000ode}, \\ \cite{devraj2017fastest} \\ \end{tabular} &
  \begin{tabular}[c]{@{}c@{}}Synchronous,\\ trajectories\end{tabular} &
  Asymptotic &
  \begin{tabular}[c]{@{}c@{}}Local, $\ell_2$ and\\ conv. in dist.\end{tabular} &
  \begin{tabular}[c]{@{}c@{}}Averaging, \\ ODE method\end{tabular} \\ \cline{3-7} 
 &
   &
   \begin{tabular}[c]{@{}c@{}}\cite{lakshminarayanan2018linear}, \cite{bhandari2018finite}, \\ \cite{srikant2019finite} \end{tabular}
 &
  \begin{tabular}[c]{@{}c@{}}Synchronous,\\ trajectories\end{tabular} &
  Non-asymptotic &
  Global, $\ell_2$ &
  \begin{tabular}[c]{@{}c@{}}Averaging, \\ martingales\end{tabular} \\ \cline{2-7} 
 &
  \multirow{2}{*}{\begin{tabular}[c]{@{}c@{}}TD-learning \\ with function\\ approximation\end{tabular}} &
  \begin{tabular}[c]{@{}c@{}}\cite{tsitsiklis1997analysis}, \\ \cite{ueno2008semiparametric} \end{tabular} &
  Trajectories &
  Asymptotic &
  \begin{tabular}[c]{@{}c@{}}Global oracle \\ inequality \\ Local, \\ conv. in dist.\end{tabular} &
  \begin{tabular}[c]{@{}c@{}}Asymptotic \\ normality \end{tabular} \\ \cline{3-7} 
 &
   &
   \begin{tabular}[c]{@{}c@{}}\cite{bhandari2018finite}, \\ \cite{doan2019finite}, \\ \cite{dalal2018finite} \end{tabular}
&
  \begin{tabular}[c]{@{}c@{}}Synchronous,\\ trajectories\end{tabular} &
  Non-asymptotic &
  \begin{tabular}[c]{@{}c@{}}Global and \\ local, $\ell_2$\end{tabular} &
  \begin{tabular}[c]{@{}c@{}}Population to\\ sample\end{tabular} \\ \cline{2-7}
&
\begin{tabular}[c]{@{}c@{}}\textbf{\red{Median of}} \\ \textbf{\red{means}} \end{tabular} &
  \textbf{\red{Current paper}} &
  \textbf{\red{Synchronous}} &
  \textbf{\red{Non-asymptotic}} &
  \textbf{\red{Local, $\ell_\infty$}} &
  \textbf{\red{Robustness}} \\ \hline
\end{tabular}%
}
\caption{A subset of results in the tabular and infinite-horizon discounted
setting, both for policy evaluation in MRPs and policy optimization in MDPs. 
For a broader overview of results, see Gosavi~\cite{gosavi2009reinforcement} for the setting of infinite-horizon average reward, and Dann and Brunskill~\cite{dann2015sample} for the episodic setting.
The ``technique'' vertical of the table is only meant to showcase a representative
subset of those employed. Our contributions are highlighted in red.}
\label{tab:related-work}
\end{table}

\paragraph{Related work:}
Markov reward processes have a rich history originating in the theory
of Markov chains and renewal processes; we refer the reader to the
classical books~\cite{Feller2} and~\cite{durrett1999essentials} for
introductions to the subject. The policy evaluation problem has seen
considerable interest in the stochastic control and reinforcement
learning communities, and various algorithms have been analyzed in
both asymptotic~\cite{borkar1998asynchronous,tadic2004almost} and
non-asymptotic~\cite{lakshminarayanan2018linear,srikant2019finite}
settings. Chapter~3 of the monograph by Szepesv\'{a}ri~\cite{Sze09}
provides a brief introduction to these methods, and the recent survey by Dann
et al.~\cite{dann2014policy} focuses on methods based on temporal
differences~\cite{sutton1988learning}.

In the language of temporal difference (TD) algorithms, the plug-in
approach that we analyze corresponds to the least squares temporal
difference (LSTD) solution~\cite{bradtke1996linear} in the tabular
setting, without function approximation.  While TD algorithms for
policy evaluation have been analyzed by many previous papers, their
focus is typically either on (i) how function approximation affects
the algorithm~\cite{tsitsiklis1997analysis}, (ii) asymptotic
convergence guarantees~\cite{borkar1998asynchronous,tadic2004almost}
or (iii) establishing convergence rates in metrics of the
$\ell_2$-type~\cite{tadic2004almost,
  lakshminarayanan2018linear,srikant2019finite}.  Since $\ell_2$-type
metrics can be associated with an inner product, many specialized
analyses can be ported over from the literature on stochastic
optimization (e.g.,~\cite{BacMou11,NemJudLanSha09}).\footnote{Here we
  have only referenced some representative papers; see the references
  in Szepesv\'{a}ri~\cite{Sze09} for a broader overview.} On the other
hand, our focus is on providing non-asymptotic guarantees in the
$\ell_{\infty}$-error metric, since these are
particularly compatible with policy iteration methods.  In particular,
policy iteration can be shown to converge at a geometric rate when
combined with policy evaluation methods that are accurate in
$\ell_\infty$-norm (e.g.,
see the books~\cite{agarwal2019reinforcement,bertsekas1996neuro}).
Also, given that we are interested in
fine-grained, instance-dependent guarantees, we first study the
problem without function approximation.

As briefly alluded to before, there has also been some recent focus on
obtaining instance-dependent guarantees in online reinforcement
learning settings~\cite{maillard2014hard,simchowitz2019non,zanette2019almost}. 
These
analyses have led to more practically applicable algorithms for certain
episodic MDPs~\cite{zanette2019tighter,jiang2018open} that
improve upon worst-case bounds~\cite{AzaOsbMun17}.  Recent work has
also established some instance-dependent bounds for the problem of
state-action value function estimation in Markov decision processes,
for both ordinary $Q$-learning~\cite{wainwright2019stochastic} and a
variance-reduced improvement~\cite{wainwright2019variance}.  However,
we currently lack the localized lower bounds that would allow us to
understand the fundamental limits of the problem in a more local
sense, except in some special cases for asymptotic settings; for
instance, see Ueno et al.~\cite{ueno2008semiparametric} and Devraj and
Meyn~\cite{devraj2017fastest} for bounds of this type for LSTD and
stochastic approximation, respectively.  We hope that our analysis of
the simpler policy evaluation problem will be useful in broadening the
scope of such guarantees.

Portions of our analysis exploit a decoupling that is induced by a
leave-one-out technique.  We note that leave-one-out techniques are
frequently used in probabilistic analysis
(e.g.,~\cite{bousquet2002stability,de2012decoupling,ma2017implicit}).
In the context of Markov processes, arguments that are related to but
distinct from those in this paper have been used in analyzing
estimates of the stationary distribution of a Markov
chain~\cite{chen2019spectral}, and for analyzing optimal policies in
reinforcement learning~\cite{agarwal2019optimality}.

For the reader's convenience, we have collected many of the relevant
results both in policy optimization and evaluation in
Table~\ref{tab:related-work}, along with the settings and sample-size
regimes in which they apply, the nature of the guarantee, and the
salient techniques used.


\paragraph{Contributions:}

We study the problem of estimating the infinite-horizon, discounted
value function of a tabular MRP in $\ell_\infty$-norm, assuming access
to state transitions and reward samples under the generative
model. Our first main result, Theorem~\ref{thm:plugin}, analyzes the
plug-in estimator, showing two types of guarantees: on one hand, we
derive high-probability upper bounds on the error that can be computed
based on the observed data, and on the other, we show upper bounds
that depend on the underlying (unknown) population transition matrix
and reward function. The latter result is achieved via a decoupling
argument that we expect to be more broadly applicable to problems of
this type.

Corollary~\ref{cor:azar} then specializes the population-based result
in Theorem~\ref{thm:plugin} to natural sub-classes of
MRPs. Theorem~\ref{thm:lb} provides minimax lower bounds for these
sub-classes, showing---in conjunction with
Corollary~\ref{cor:azar}---that the plug-in approach is minimax
optimal over the class of MRPs with uniformly bounded reward
functions. However, these results suggest that the plug-in approach is
\emph{not} minimax-optimal over the class of MRPs having value
functions with bounded variance under the transition model (this
notion is defined precisely in Section~\ref{sec:main-results} to
follow).  Consequently, we analyze an approach based on the
median-of-means device and show that this modified estimator is
minimax optimal over the class of MRPs having value functions with
bounded variance.

\begin{figure}[ht]
  \begin{centering}
  \includegraphics[scale = 0.4]{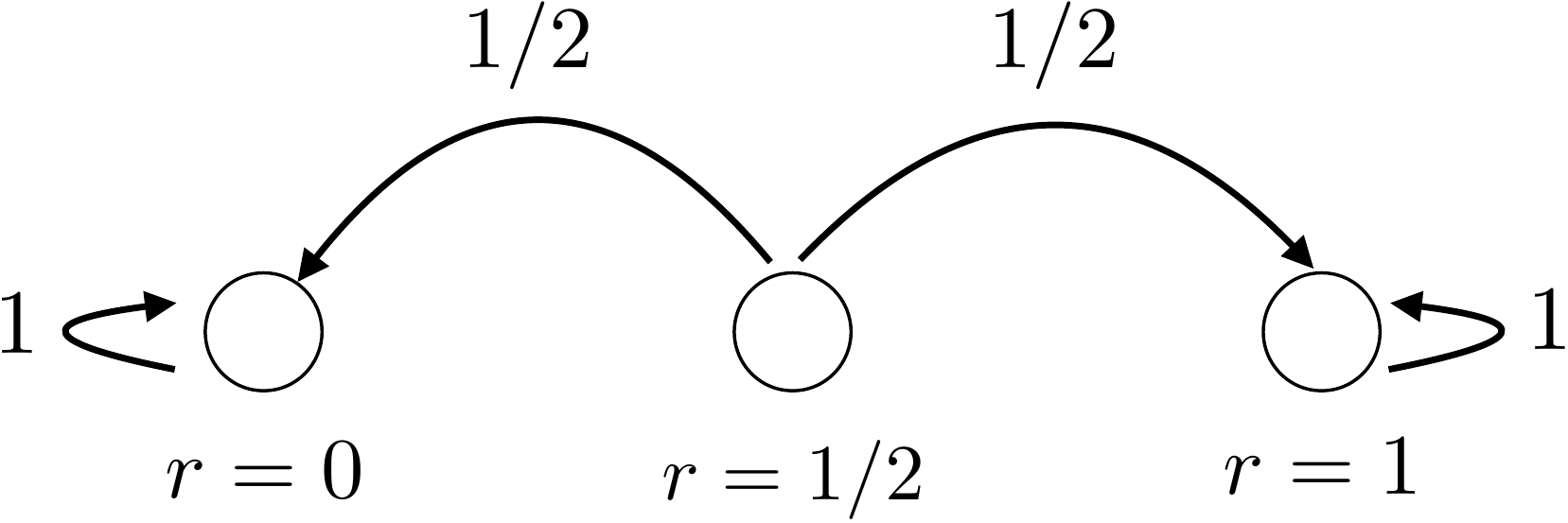}
  \caption{A simple $3$-state Markov reward process.}
  \label{fig:simple}
  \end{centering}
\end{figure}
The benefits of our instance-dependent guarantees are even evident in
a model as simple as the $3$-state MRP illustrated in
Figure~\ref{fig:simple}.  Suppose that we observe noiseless rewards of
this MRP and wish to compute its infinite-horizon value function with
discount factor $\discount \in (0, 1)$. Bounds based on the
contractivity of the Bellman
operator~\cite{KeaSin99,Kak03,wainwright2019stochastic} imply that the
$\ell_\infty$-error of the plug-in estimate scales proportionally to
$1 / (1 - \discount)^{2}$. The worst-case bounds of Azar et
al.~\cite{AzaMunKap13} imply a rate $1 / (1 - \discount)^{3/2}$. But
the optimal local result captured in this paper shows that the error
is only proportional to $1 / (1 - \discount)$. For a discount factor
$\discount = 0.99$, this improves the previous bounds by factors of
$100$ and $10$, respectively, and consequently, the respective sample
complexities by factors of $10^4$ and $10^2$. Instance-dependent
results therefore allow us to differentiate problems that are
``solvable'' with finite samples from those that are not.

\paragraph{Notation:}
For a positive integer $n$, let $[n] \defn \{1, 2, \ldots, n\}$. For a
finite set $S$, we use $|S|$ to denote its cardinality.  We use $c, C,
c_1, c_2, \dots$ to denote universal constants that may change from
line to line.  We use the convenient shorthand $a \lor b \defn \max\{
a, b\}$ and $a \land b = \min\{ a, b \}$.  We let $\ones$ denote the
all-ones vector in $\real^{\Dim}$, and abusing notation slightly, we
let $\ind{\Espace}$ denote the indicator of an event~$\Espace$. Let
$e_j$ denote the $j$th standard basis vector in $\real^{\Dim}$.  We
let $v_{(i)}$ denote the $i$-th order statistic of a vector $v$, i.e.,
the $i$-th largest entry of~$v$.  For a pair of vectors $(u, v)$ of
compatible dimensions, we use the notation $u \preceq v$ to indicate
that the difference vector $v - u$ is entry-wise non-negative. The
relation $u \succeq v$ is defined analogously. We let $|u|$ denote the
entry-wise absolute value of a vector $u \in \real^{\Dim}$; squares
and square-roots of vectors are, analogously, taken entrywise. Note
that for a positive scalar $\lambda$, the statements $|u | \preceq
\lambda \cdot \ones$ and $\| u \|_{\infty} \leq \lambda$ are
equivalent. Finally, we let $\| \Mmat \|_{1, \infty}$ denote the
maximum $\ell_1$-norm of the rows of a matrix $\Mmat$, and refer to it
as the $(1, \infty)$-operator norm of a matrix.


\section{Background and Problem Formulation}
\label{SecBackground}

In this section, we introduce the basic notation required to specify a
Markov reward process, and formally define the problem of estimating
value functions in the generative setting.

\subsection{Markov reward processes and value functions}

We study Markov reward processes defined on a finite set $\StateSpace$
of states, indexed as $\StateSpace = \{1, 2, \ldots, \Dim \}$. The
state evolution over time is determined by a set of transition
functions $\{ P(\; \cdot \; \mid \state), \; \; \state \in \StateSpace
\big\}$, with the transition from state $\state$ to the next state
being randomly chosen according to the distribution $P( \, \cdot \,
\mid \state)$.  For notational convenience, we let $\Pmat \in [0
  ,1]^{\Dim \times \Dim}$ denote a row stochastic (Markov) transition
matrix, where row $j$ of this matrix---which we denote by
$\pvec{j}$---collects the transition function of the $j$-th state.
Also associated with an MRP is a \emph{population} reward function
$\reward: \StateSpace \mapsto \real$: transitioning from state
$\state$ results in the reward $\reward (\state)$. For convenience, we
engage in a minor abuse of notation by letting $\reward$ also denote a
vector of length $\Dim$, with $\reward_j$ corresponding to the reward
obtained at the $j$-th state.

In this paper, we consider the infinite-horizon, discounted reward as
our notion for the long-term value of a state in the MRP. 
In particular, for a scalar discount factor $\discount \in (0, 1)$,
the long-term value of state $\state$ in the MRP is given by
\begin{align*}
\ValueFunc^* (\state) \defn \EE \left[ \sum_{k = 0}^\infty \discount^k
  \reward(\state_k) \; \st \; \state_0 = \state \right], \qquad \text{
  where } \state_k \sim P(\; \cdot \; | \; \state_{k - 1}) \text{ for all
} k \geq 1.
\end{align*} 
In words, this measures the expected discounted reward obtained by
starting at the state $\state$, where the expectation is taken with
respect to the random transitions over states.  Once again, we use
$\ValueFunc^*$ to also denote a vector of length $\Dim$, where
$\ValueFunc^*_j$ corresponds to the value of the $j$-th state.

A note to the reader: in the sequel, we often reference a state simply
by its index, and often refer to the state space $\Xspace \equiv
[\Dim]$. Accordingly, we also use $P(\; \cdot \; \mid j)$ to denote
the transition function corresponding to state $j \in [\Dim]$.


\subsection{Observation model}

Given access to the true transition and reward functions, it is
straightforward, at least in principle, to compute the value function.
By definition, it is the unique solution of the Bellman fixed point
relation
\begin{align}
\label{eq:Bellman-true}
\ValueFunc^* = \reward + \discount \Pmat \ValueFunc^*.
\end{align} 
In the learning setting, the pair $(\Pmat, \reward)$ is unknown, and
we instead assume access to a black box that generates samples from
the transition and reward functions.  In this paper, we operate under
a setting known as the \emph{synchronous or generative setting}; it is
a stylized observation model that has been used extensively in the
study of Markov decision processes (see Kearns and Singh~\cite{KeaSin99}
for an introduction).  Let us introduce
it in the context of MRPs: for a given sample index $k = 1, 2, \ldots,
\Nsamp$ and for each state $j \in [\Dim]$, we observe a random next
state $X_{k,j} \in [\Dim]$ drawn according to the transition function
$P( \; \cdot \; \mid \; j)$, and a random reward $\Reward_{k,j}$ drawn
from a conditional distribution $\rdist(\, \cdot \, \mid j)$.
Throughout, we assume that the rewards are generated independently
across states, with $\EE[ \Reward_{k,j} ] = \rstar_j$.
Letting
$\sigmar$ denote a non-negative vector indexed by the states $j \in
           [\Dim]$, we assume the conditional distributions $\{
           \rdist(\, \cdot \, \mid j), \; j \in [\Dim] \}$ are
           $\sigmar$-sub-Gaussian, meaning that for each $j \in
                 [\Dim]$, we have
\begin{align}
  \label{eq:reward-sg}
\EE_{R \sim \rdist(\, \cdot \, \mid j)} \Big[ e^{\lambda (R -
    \rstar_j)} \Big] \leq e^{ \frac{\lambda^2 \sigmarsqindex}{2}}
\qquad \mbox{for all $\lambda \in \real$.}
\end{align}
With $\Nsamp$ such i.i.d. samples in hand, our goal is to estimate the value
function $\ValueFunc^*$ in the
$\ell_\infty$-error metric.

Such a goal is particularly relevant to the policy evaluation
problem described in the introduction, since
$\ell_\infty$-estimates of the value function can be used in conjunction with a policy
improvement sub-routine to eventually arrive at an optimal
policy (see, e.g., Section 1.2.2. of the recent monograph~\cite{agarwal2019reinforcement}). 
We note in passing that
bounds proved under the generative model may be translated into the
more challenging online setting via the notion of Markov cover times
(see, e.g., the papers~\cite{EveMan03,Aza11} for conversions of this
type for Markov decision processes).


\section{Main results}
\label{sec:main-results}

We now turn to the statement and discussion of our main results.  We
begin by providing \mbox{$\ell_\infty$-guarantees} on value function
estimation for the natural plug-in approach.

\subsection{Guarantees for the plug-in approach}

A natural approach to this problem is use the observations to
construct estimates $(\Phat, \rhat)$ of the pair $(\Pmat, \rstar)$,
and then substitute or ``plug in'' these estimates into the Bellman
equation, thereby obtaining the value function of the MRP having
transition matrix $\Phat$ and reward vector $\rhat$.

In order to define the plug-in estimator, let us introduce some
helpful notation.  For each time index $k$, we use the associated set
of state samples $\{X_{k,j}, j \in [\Dim] \}$ to form a random binary
matrix $\Zmatsub{k} \in \{0,1\}^{\Dim \times \Dim}$, in which row $j$
has a single non-zero entry, determined by the sample $X_{k,j}$.
Thus, the location of the non-zero entry in row $j$ is drawn from the
probability distribution defined by $\pvec{j}$, the $j$-th row of
$\Pmat$.  Recall that our observations also include the stochastic
reward vectors $\{\Reward_k\}_{k=1}^\Nsamp$ sampled from the reward
distribution $\rdist$.  Based on these observations, we define the
sample means
\begin{align}
\label{eq:plugin}
\Phat = \frac{1}{\Nsamp} \sum_{k = 1}^{\Nsamp} \Zmatsub{k} \quad
\text{ and } \quad \rhat = \frac{1}{\Nsamp} \sum_{k = 1}^{\Nsamp}
\Reward_k,
\end{align}
which can be seen as unbiased estimates of the transition matrix $\Pmat$
and the reward vector $\reward$, respectively.

The estimates $(\Phat, \rhat)$ define a new MRP, and its value
function is given by the fixed point relation
\begin{align}
  \label{eq:Bellman-est}
\ValueFunchat = \rhat + \discount \Phat \ValueFunchat.
\end{align}
Solving this fixed point equation, we obtain the closed form
expression $\ValueFunchat = (\Id - \discount \Phat)^{-1} \rhat$ for
the plug-in estimator.  Note that the terminology ``plug-in'' arises
the fact that $\ValueFunchat$ is obtained by substituting the
estimates $(\Phat, \rhat)$ into the original Bellman
equation~\eqref{eq:Bellman-true}.  We also note that in this special
case---that is, the tabular setting without function
approximation---the plug-in estimate is equivalent to the LSTD
solution~\cite{bradtke1996linear,boyan2002technical}.

In order to establish guarantees for the estimator $\ValueFunchat$, we
require some additional notation.  As mentioned before, we are
interested in non-asymptotic, instance-dependent guarantees of two
types: the first is a bound that can be evaluated in practice from the
observed data, and the second is a guarantee that depends on the
unknown population quantities $\Pmat$ and $\reward$.  For each vector
$\theta \in \real^{\Dim}$, define the vector of empirical variances
\begin{align*}
\sighat^2(\theta) & = \Ehat \left| (\Zmat - \Phat) \theta \right|^2,
\end{align*}
where $\Ehat$ denotes expectation over the empirical distribution
(i.e., the random matrix $\Zmat$ is drawn uniformly at random from the
set $\{\Zmatsub{k}\}_{k=1}^\Nsamp$).  Note that given $\theta$, this
quantity is computable purely from the observed samples.
On the other hand, the population result will involve the population
variance vector
\begin{align*}
\sigma^2(\theta) & = \Exs \left| (\Zmat - \Pmat) \theta
\right|^2,
\end{align*}
where in this case $\Zmat$ is drawn according to the population model
$\Pmat$.  As a final definition, the span semi-norm of a value
function $\theta$ is given by
\begin{align*}
\| \theta \|_{\myspan} \defn \max_{\state \in \StateSpace}
\theta(\state) - \min_{\state \in \StateSpace} \theta(\state).
\end{align*}
Equivalently, the span semi-norm is equal to the variation of the
vector $\theta \in \real^{\Dim}$; see Puterman~\cite{Puterman05} for
more details. \\

\noindent We now ready to state our main result for the plug-in
estimator.
\begin{theorem} \label{thm:plugin}
There is a pair of universal constants $(\unicon_1, \unicon_2)$ such
that if $\Nsamp \geq \unicon_1 \frac{
  \discount^2}{(1-\discount)^2} \log(8 \Dim/\delta)$, then
  each of the following statements holds with probability at least $1 -
\delta$.
\begin{enumerate}
\item[(a)] We have
  \begin{subequations}
  \label{eq:thm1}
    \begin{align}
      \label{eq:thm1-data}
\hspace{-8mm} \|\ValueFunchat - \ValueFunc^* \|_\infty &\leq \unicon_2
\left\{ \sqrt{\frac{\log(8 \Dim/\delta)}{\Nsamp}} \left( \discount
\|(\IdMat - \discount \Phat)^{-1} \sighat (\ValueFunchat)\|_\infty +
\frac{\|\sigmar\|_\infty}{ 1 - \discount} \right) + \frac{\log(8D /
  \delta)}{N} \cdot \frac{\discount \| \ValueFunchat \|_{\myspan}}{1 -
  \discount} \right\}.
    \end{align}
\item[(b)] We have
  \begin{align}
    \label{eq:thm1-pop}
\hspace{-8mm} \|\ValueFunchat - \ValueFunc^* \|_\infty & \leq
\unicon_2 \left\{ \sqrt{\frac{\log(8D / \delta)}{N} } \left( \discount
\left \| (\IdMat - \discount \Pmat)^{-1} \sigma(\ValueFunc^*)
\right\|_\infty + \frac{\|\sigmar\|_\infty}{1 - \discount} \right) +
\frac{\log(8D / \delta)}{N} \cdot \frac{\discount \| \ValueFunc^*
  \|_{\myspan}}{1 - \discount} \right\}.
  \end{align}
  \end{subequations}
\end{enumerate}
\end{theorem}
It is worth making a few comments on this theorem, which provides two
instance-dependent upper bounds on the error of the plug-in
approach. Assuming for simplicity of discussion\footnote{We note that
  when $\sigmar$ is \emph{not} known but the reward distribution is
  (say) Gaussian, it is straightforward to provide an entry-wise upper
  bound for it by computing the empirical standard deviation of
  rewards from samples, and using this to define a high-probability
  and data-dependent bound on the sub-Gaussian parameter.}  that the
maximum noise reward parameter $\| \sigmar \|_\infty$ is known, then
part (a) of the theorem provides a bound that can be evaluated based
on the observed data; bounds of this form are especially useful in
downstream analyses.  For instance, a central consideration in policy
iteration methods is to obtain ``good enough'' value function
estimates $\thetahat$ for fixed policies, in that we have $\|
\thetahat - \thetastar \|_{\infty} \leq \epsilon$ for some prescribed
tolerance~$\epsilon$.  Theorem~\ref{thm:plugin}(a) provides a method
by which such a bound may be verified for the plug-in approach:
compute the statistic on the RHS of bound~\eqref{eq:thm1-data}; if
this is less than $\epsilon$, then the bound $\| \ValueFunchat -
\thetastar \|_{\infty} \leq \epsilon$ holds with probability exceeding
$1 - \delta$.

On the other hand, Theorem~\ref{thm:plugin}(b) provides a guarantee
that depends on the unknown problem instance. From the perspective of
the analysis, this is the more difficult bound to establish, since it
requires a leave-one-out technique to decouple dependencies between
the estimate $\ValueFunchat$ and the matrix~$\Phat$.  
We expect our technique---presented in full in
Section~\ref{sec:pf-thm1b}---and its variants to be more broadly
useful in analyzing other problems in reinforcement learning besides
the policy evaluation problem considered here.

Third, note that our lower bound on the sample size---which evaluates
to $\Nsamp \geq \frac{\unicon_1}{(1-\discount)^2} \log(8 \Dim/\delta)$
for any strictly positive discount factor---is unavoidable in general.
In particular, for any fixed reward-noise parameter $\|\sigmar
\|_\infty > 0$, this condition is required in order to obtain a
consistent estimate of the value function.\footnote{For instance, even
  with known transition dynamics, estimating the value function of a
  single state to within additive error $\epsilon$ requires $\Omega
  \left(\tfrac{1}{(1 - \discount)^2 \epsilon^2}\right)$ samples of the
  noisy reward.}  On the other hand, in the special case of
deterministic rewards ($\|\sigmar\|_\infty = 0$), we suspect that this
condition can be weakened, but leave this for future work.

Finally, it is worth noting that there are two terms in the bounds of
Theorem~\ref{thm:plugin}: the first term corresponds to a notion of
standard deviations of the estimated/true value function and reward,
and the second depends on the span semi-norm of the value
function. Are both of these terms necessary? What is the optimal rate
at which any value function can be estimated?  These questions
motivate the analysis to be presented in the following section.


\subsection{Is the plug-in approach optimal?}

In order to study the question of optimality, we adopt the notion of
\emph{local minimax risk}, in which the performance of an estimator is
measured in a worst-case sense locally over natural subsets of the
parameter space.  Our upper bounds depend on the problem instance via
the standard deviation function $\sigma(\thetastar)$, the reward
standard deviation $\sigmar$, and the span semi-norm of $\thetastar$.
Accordingly, we define the following subsets\footnote{The following mnemonic device
may help the reader appreciate and remember notation:
the symbol $\vartheta$, or ``vartheta'', stands for a measure of the variability in the
value function $\theta$; the symbol $\varrho$, or ``varrho'', represents
the variability in reward samples, and $\rmax$ represents the maximum absolute reward mean.} of Markov reward
processes (MRPs):
\begin{subequations}
  \begin{align}
  \MspaceVar(\sigvalbd, \sigrewbd) & \defn \Big \{ \mbox{set of all MRPs s.t.
    $\|\sigma(\thetastar)\|_\infty \leq \sigvalbd$ and $\|\sigmar\|_\infty
    \leq \sigrewbd$} \Big \}, \\
  \MspaceValFun(\valbd, \sigrewbd) & \defn \Big \{ \mbox{set of all MRPs s.t.
    $\spannorm{\thetastar} \leq \valbd$ and $\|\sigmar\|_\infty \leq \sigrewbd$} \Big \}, \qquad \mbox{and}  \\
  \MspaceRew(\rmax, \sigrewbd) & \defn \Big \{ \mbox{set of all MRPs s.t.
    $\|\reward\|_\infty \leq \rmax$ and $\|\sigmar\|_\infty \leq \sigrewbd$} \Big \}.
  \end{align}
\end{subequations}
Letting $\Mspace$ be any one of these sets, 
we use the shorthand
$\theta \in \Mspace$ to mean that $\theta$ is the value function of
some MRP in the set $\Mspace$.  Each choice of the set $\Mspace$
defines the local minimax risk given by
\begin{align*}
\inf_{\thetahat} \; \sup_{ \ValueFunc^* \in \Mspace } \; \EE \left[\|
  \thetahat - \thetastar \|_\infty \right],
\end{align*}
where the infimum ranges over all measurable functions $\thetahat$ of
$\Nsamp$ observations from the generative model.  With this set-up, we
can now state some lower bounds in terms of such local minimax risks:
\begin{theorem} \label{thm:lb}
  There is a pair of absolute constants $(\unicon_1, \unicon_2)$ such
  that for all $\discount \in [\tfrac{1}{2}, 1)$ and sample sizes
    $\Nsamp \geq \frac{\unicon_1}{1 - \discount} \log
    (\Dim/2)$, the following statements hold.
\begin{enumerate}
\item[(a)] For each triple
  of positive scalars $(\sigvalbd, \valbd, \sigrewbd)$ satisfying\footnote{We
  conjecture that this lower bound can be proved under the weaker condition
  $\sigvalbd \leq \valbd$, thereby matching the condition present in Corollary~\ref{cor:azar}(a).}
  $\sigvalbd \leq \valbd \sqrt{1 - \discount}$, we have
\begin{subequations}
\begin{align} \label{eq:thm2-var}
\inf_{\thetahat} \; \sup_{ \ValueFunc^*
  \in \Mspace_{\var}(\sigvalbd, \sigrewbd) \cap \Mspace_{\val}(\valbd, \sigrewbd)} \; \EE \left[\| \thetahat
  - \ValueFunc^* \|_{\infty} \right] \geq \frac{\unicon_2}{1 -
  \discount} \sqrt{\frac{\log(\Dim/2 )}{\Nsamp}} \left( \sigvalbd +
\sigrewbd \right).
\end{align}
\item[(b)] For each pair of positive scalars $(\rmax, \sigrewbd)$
  satisfying $\rmax \geq \sigrewbd \sqrt{\frac{\log \Dim}{\Nsamp}}$, we
  have
\begin{align}
\label{eq:thm2-rew}
\inf_{\thetahat} \; \sup_{ \ValueFunc^* \in \Mspace_{\rew} (\rmax,
  \sigrewbd) } \; \EE \left[\| \thetahat - \ValueFunc^* \|_{\infty}
  \right] \geq \frac{\unicon_2}{1 - \discount}
\sqrt{\frac{\log(\Dim/2)}{\Nsamp}} \left( \frac{\rmax}{(1 -
  \discount)^{1/2}} + \sigrewbd \right).
\end{align}
\end{subequations}
\end{enumerate}
\end{theorem}

Equipped with these lower bounds, we can now assess the local minimax
optimality of the plug-in estimator.  In order to facilitate this
comparison, let us state a corollary of Theorem~\ref{thm:plugin} that
provides bounds on the worst-case error of the plug-in estimator over
particular subsets of the parameter space.  In order to further
simplify the comparison, we restrict our attention to the range
$\discount \in [\tfrac{1}{2}, 1)$ covered by the lower bounds.
\begin{corollary}
\label{cor:azar}
There are absolute constants $(\unicon_3, \unicon_4)$ such that 
for all $\discount \in [ \tfrac{1}{2}, 1)$ and sample sizes\footnote{As
shown in the proof, part (a) of the corollary holds without this assumption
on the sample size, but we state it here to facilitate a direct derivation of
Corollary~\ref{cor:azar} from Theorem~\ref{thm:plugin}.
}
 $\Nsamp \geq \frac{\unicon_3}{(1 - \discount)^2} \log (8
\Dim / \delta)$, the following statements hold.
\begin{enumerate}
\item[(a)] Consider a triple of positive scalars $(\sigvalbd, \valbd,
  \sigrewbd)$ such that\footnote{It is worth noting that the condition
    $\sigvalbd \leq \valbd$ in part (a) of the corollary does not
    entail any loss of generality, since we always have $\| \sigma
    (\thetastar) \|_\infty \leq \| \thetastar \|_{\myspan}$. Indeed,
    for MRPs in which \mbox{$\| \sigma (\thetastar) \|_\infty \ll \|
      \thetastar \|_{\myspan}$}, the second term on the RHS of
    inequality~\eqref{eq:cor1-var} will dominate the bound unless the
    sample size $\Nsamp$ is large.}  $\sigvalbd \leq \valbd$. Then for
  any value function \mbox{$\ValueFunc^* \in \MspaceVar(\sigvalbd,
    \sigrewbd) \cap \MspaceValFun(\valbd, \sigrewbd)$}, we have
  \begin{subequations}
    \begin{align} \label{eq:cor1-var}
      \|\ValueFunchat - \ValueFunc^* \|_\infty & \leq \frac{\unicon_4
      }{1 - \discount} \left\{ \sqrt{\frac{\log(8D / \delta)}{N}}
      \left( \sigvalbd + \sigrewbd \right) + \frac{\log(8D /
        \delta)}{N} \cdot \valbd \right\}
    \end{align}
with probability at least $1 - \delta$.
  \item[(b)] Consider an arbitrary pair of positive scalars $(\rmax,
    \sigrewbd)$.  Then for any value function \mbox{$\ValueFunc^*
      \in \Mspace_{\rew}(\rmax, \sigrewbd)$}, we have
    \begin{align} \label{eq:cor1-rew}
      \|\ValueFunchat - \ValueFunc^* \|_\infty &\leq \frac{\unicon_4
      }{1 - \discount} \sqrt{\frac{\log(8D / \delta)}{N}} \left(
      \frac{ \rmax}{(1 - \discount)^{1/2}} + \sigrewbd \right)
    \end{align}
 with probability at least $1- \delta$.
  \end{subequations}
\end{enumerate}
\end{corollary}

By comparing Corollary~\ref{cor:azar}(b) with Theorem~\ref{thm:lb}(b),
we see that the plug-in estimator is minimax optimal (up to constant
factors) over the class $\Mspace_{\rew}(\rmax, \sigrewbd)$.  This
conclusion parallels that of Azar et al.~\cite{AzaMunKap13} for the
related problem of optimal state-value function estimation in MDPs.
(In our notation, their work applies to the special case of $\sigrewbd
= 0$, but their analysis can easily be extended to this more general
setting.)

A comparison of part (a) of the two results is more interesting.  Here
we see that the first term in the upper bound~\eqref{eq:cor1-var}
matches the lower bound~\eqref{eq:thm2-var} up to a constant
factor. The second term of
inequality~\eqref{eq:cor1-var}, however, does not have an analogous
component in the lower bound, and this leads us to the interesting
question of whether the analysis of the plug-in estimator can be
sharpened so as to remove the dependence of the error on the span
semi-norm $\| \thetastar \|_{\myspan}$. Proposition~\ref{prop:plugin-lb},
presented in Appendix~\ref{AppPluginLB}, 
shows that this is impossible in general,
and that there are MRPs in which the $\ell_\infty$ error can be
\emph{lower bounded} by a term that is proportional to the span semi-norm.

This raises another natural question: Is there
a different estimator whose error can be bounded independently
of the span semi-norm $\| \thetastar \|_{\myspan}$, and which
is able achieve the lower bound~\eqref{eq:thm2-var}?
In the next section, we introduce such an
estimator via a median-of-means device.


\subsection{Closing the gap via the median-of-means method}
\label{sec:MoM}

In many situations, the span semi-norm of a value function
$\ValueFunc^*$ may be much larger its variance~$\sigma(\thetastar)$
under the transition model.  Such a discrepancy arises when there are
states with extremely large positive (or negative) rewards that are
visited with very low probability.  In such cases, the second terms in
the bounds~\eqref{eq:thm1} dominate the first. It is thus of interest
to derive bounds that are purely ``variance-dependent'' and
independent of the span norm.  In order to do so, we analyze a slight
variant of the plug-in approach.  In particular, we analyze the
\emph{median-of-means} estimator, which is a standard robust
alternative to the sample mean in other
scenarios~\cite{NemYu83,lecue2017robust}.  In the context of
reinforcement learning, Pazis et al.~\cite{pazis2016improving} made
use of it for online policy optimization in MDPs.

In our setting, we only employ median-of-means to obtain a better
estimate of term depending on the transition matrix; we still use the
estimate $\rhat$ defined in equation~\eqref{eq:plugin} as our estimate
of the reward function.\footnote{In principle, one could run a
  median-of-means estimate on the combination of reward and
  transition, but this is not necessary in our setting due to the
  sub-Gaussian assumption on the reward
  noise~\eqref{eq:reward-sg}. Slight modifications of our techniques
  also yield bounds for the combined median-of-means estimate assuming
  only that the standard deviation of the reward noise is bounded
  entry-wise by the vector $\sigmar$.}  Given the data set $\left\{
\Zmatsub{k} \right\}_{k = 1}^{\Nsamp}$ and some vector $\theta \in
\real^{\Dim}$, the median-of-means estimate~$\MoM(\theta)$ of the
population expectation $\Pmat \theta$ is given by the following
nonlinear operation:
\begin{itemize}
  \item First, split the data set into $K$ equal parts denoted
    $\{ \Dspace_1, \ldots, \Dspace_K \}$, where each subset
    $\Dspace_i$ has size $m = \lfloor \Nsamp / K \rfloor$.
  \item Second, compute the empirical mean $\muhat_i(\theta) \defn
    \frac{1}{m} \sum_{k \in \Dspace_i} \Zmatsub{k} \theta$ for each $i
    \in [K]$.
  \item Finally, return the quantity~$\MoM(\theta) \defn \med(
    \muhat_1(\theta), \ldots, \muhat_K(\theta))$, where the
    median---defined for convenience as the~$\lfloor K/ 2 \rfloor$-th
    order statistic---is taken entry-wise.
\end{itemize}

The random operator $\MoM$ defines the \emph{median-of-means empirical
  Bellman operator}, given by
\begin{align}
\label{eq:Bellman-MoM}
\Belemp^{\robust}_\Nsamp (\theta) & \defn \rhat + \discount
\MoM(\theta).
\end{align}
As shown in Lemma~\ref{lem:op-contracts} (see
Section~\ref{sec:proofs}), this operator is $\discount$-contractive in
the $\ell_\infty$-norm.  Consequently, it has a unique fixed point,
which we term the \emph{median-of-means value function estimate},
denoted by $\ValueFunchatrob$.

In practice, the estimate $\ValueFunchatrob$ can be found by starting
at an arbitrary initialization and repeatedly applying the
$\discount$-contractive operator $\Belemp^{\robust}_{\Nsamp}$ until
convergence.\footnote{Since the operator is $\discount$-contractive,
  it suffices to run this iterative algorithm for $\log_{\discount}
  \epsilon$ to obtain an $\epsilon$-approximate fixed point in an
  additive sense.}  The following theorem provides a population-based
guarantee on the error of this estimator.
\begin{theorem}
  \label{thm:rob}
Suppose that the median-of-means operator $\MoM$ is constructed with
the parameter choice \mbox{$K = 8 \log (4 \Dim / \delta)$}. Then there is a
universal constant $\unicon$ such that we have
\begin{align}
\label{eq:robust-result}
\|\ValueFunchatrob - \ValueFunc^* \|_\infty &\leq \frac{\unicon}{1 -
  \discount} \sqrt{\frac{\log(8D / \delta)}{N} } \Big( \discount \|
\sigma(\ValueFunc^*) \|_\infty + \|\sigmar\|_\infty \Big)
\end{align}
with probability exceeding $1 -\delta$.
\end{theorem}

We have thus achieved our goal of obtaining a purely
variance-dependent bound. Indeed, for each pair of positive scalars
$(\sigvalbd, \sigrewbd)$, any value function~\mbox{$\thetastar
  \in \Mspace_{\var}(\sigvalbd, \sigrewbd)$}, and reward distribution satisfying
$\| \sigmar \|_\infty \leq \sigrewbd$, we have
\begin{align*}
\|\ValueFunchatrob - \ValueFunc^* \|_\infty \leq \frac{\unicon}{1 -
  \discount} \sqrt{\frac{\log(8D / \delta)}{N} } \left( \sigvalbd +
\sigrewbd \right),
\end{align*}
with probability exceeding $1 - \delta$. Integrating this tail bound
yields an analogous upper bound on the expected error, which matches
the lower bound~\eqref{eq:thm2-var} on the expected error up to a
constant factor.  As a corollary, we conclude that the minimax risk
over the class  $\Mspace_{\var}(\sigvalbd, \sigrewbd)$ scales as
\begin{align}
\inf_{\thetahat} \; \sup_{ \ValueFunc^* \in \Mspace_{\var}(\sigvalbd,
  \sigrewbd) } \; \EE \left[\| \thetahat - \ValueFunc^* \|_{\infty}
  \right] \asymp \frac{1}{1 - \discount} \sqrt{\frac{\log(\Dim
    )}{\Nsamp}} \left( \sigvalbd + \sigrewbd \right),
\end{align}
and is achieved (up to constant factors) by the estimator
$\ValueFunchatrob$.

However, our results fall short of showing that the estimator
$\ValueFunchatrob$ is minimax optimal over the class
$\MspaceRew(\rmax, \sigrewbd)$ of MRPs with bounded rewards. Indeed,
for any value function $\thetastar$ in the class
\mbox{$\MspaceRew(\rmax, \sigrewbd)$,} Theorem~\ref{thm:rob} yields
the corollary
\begin{align*}
\|\ValueFunchatrob - \ValueFunc^* \|_\infty \leq \frac{\unicon}{1 -
  \discount} \sqrt{\frac{\log(8D / \delta)}{N} } \left( \discount
\frac{\rmax}{1 - \discount} + \sigrewbd \right)
\end{align*}
with probability exceeding $1 - \delta$. Comparing
inequality~\eqref{eq:thm2-rew} with this bound, we see that our upper
bound on the median-of-means estimator is sub-optimal by a factor
\mbox{$(1 - \discount)^{-\tfrac{1}{2}}$} in the discount complexity.
From a technical standpoint, this is due to the fact that our upper
bound in Theorem~\ref{thm:rob} involves the functional $\frac{1}{1 -
  \discount}\| \sigma(\thetastar) \|_{\infty}$ and not the sharper
functional $\| (\Id - \discount \Pmat)^{-1} \sigma(\thetastar)
\|_{\infty}$ present in Theorem~\ref{thm:plugin}(b). We believe that
this gap is not intrinsic to the MoM method, and conjecture that an
upper bound depending on the latter functional can be proved for the
estimator $\ValueFunchatrob$; this would guarantee that the
median-of-means estimator is also minimax optimal over the class
$\Mspace_{\rew} (\rmax, \sigrewbd)$.

\section{Numerical experiments}

In this section, we explore the sharpness of our theoretical
predictions, for both the plug-in and the median-of-means (MoM)
estimator. Our bounds predict a range of
behaviors depending on the scaling of the maximum standard deviation
$\|\sigma(\thetastar)\|_\infty$, and the span semi-norm (for the
plug-in estimator). Let us verify these scalings via some simple experiments.  

\subsection{Behavior on the ``hard" example used for the lower bound}

First, we use a simple variant of our lower bound construction illustrated in panel (a) of Figure~\ref{fig:mrp}. This MRP consists of $\Dim = 2$ states, where state $1$ stays fixed with probability $p$, transitions to state $2$ with
probability $1-p$, and state $2$ is absorbing.  The rewards in states
$1$ and $2$ are given by $\nu$ and $\nu \tau$, respectively. Here the
triple $(p, \nu, \tau)$, along with the discount factor $\discount$,
are parameters of the construction.

\begin{figure}[ht]
  \begin{center}
    \begin{tabular}{ccc}
          \raisebox{1in}{\widgraph{0.35\textwidth}{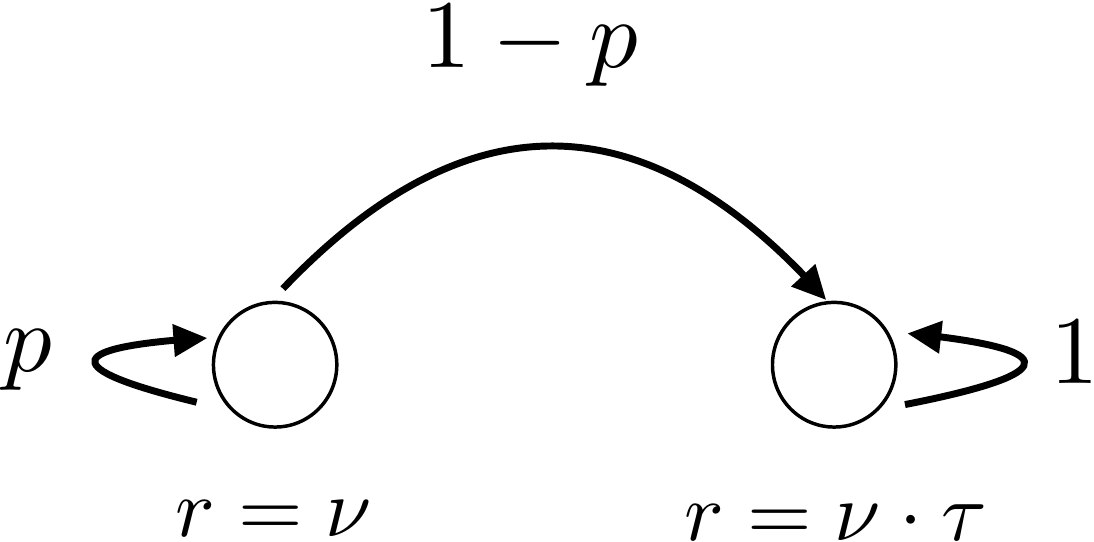}} &&
          \widgraph{0.58\textwidth}{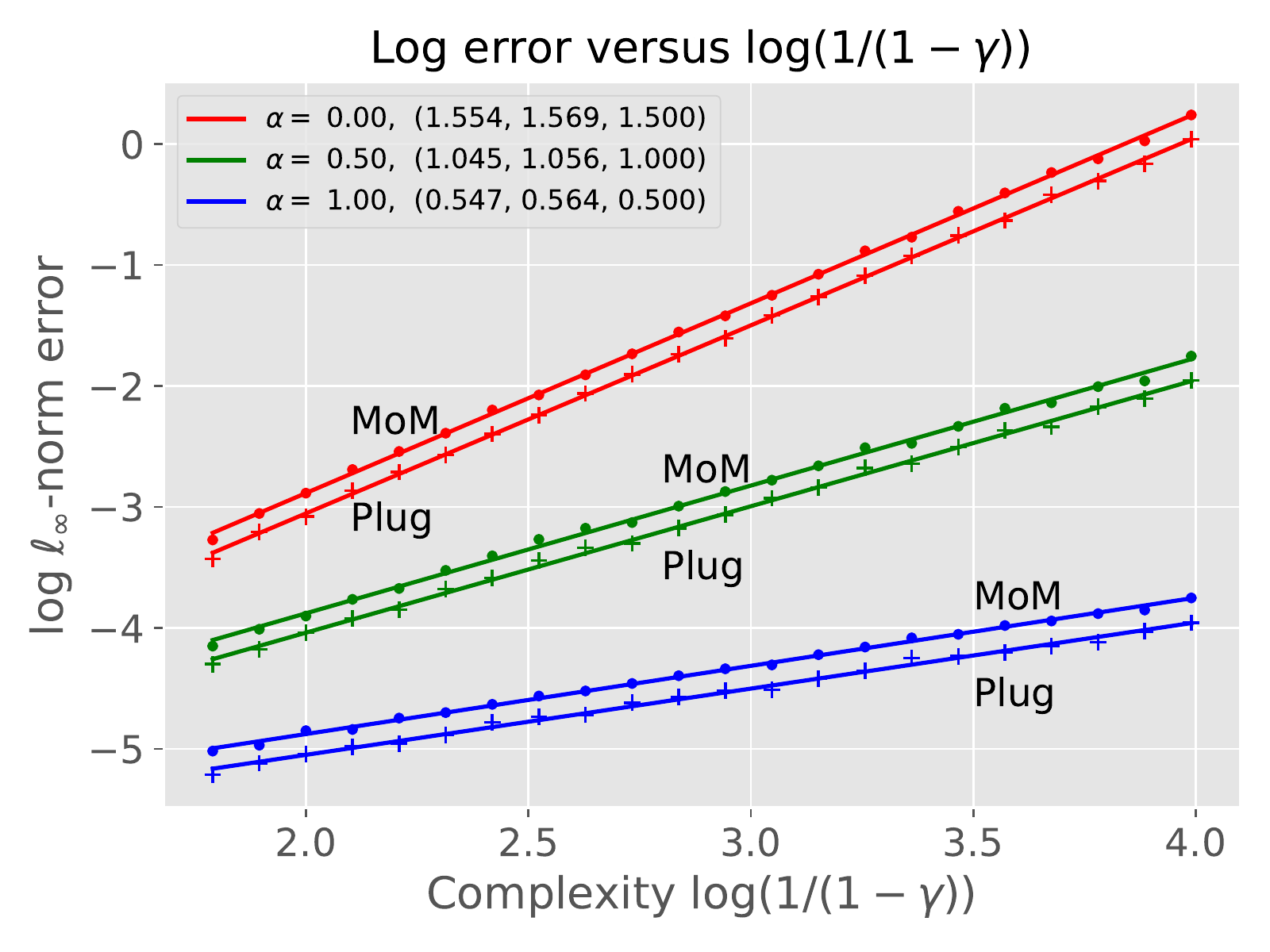}
          \\ (a) & & (b) \\
          && 
    \end{tabular}
    \caption{(a) Illustration of the MRP $\MRP_{\basic} (p, \lbmaxrew,
      \tau)$ used in the simulation, and also as a building block in
      the lower bound construction of Theorem~\ref{thm:lb}.  For the
      simulation, we choose $p = \tfrac{4\discount - 1}{3\discount}$,
      let $\lbmaxrew = 1$, and set $\tau = 1 - (1 -
      \discount)^{\alpha}$.  (b) Log-log plot of the
      $\ell_\infty$-error versus the discount complexity parameter $1
      / (1 - \discount)$ for both the plug-in estimator (in $+$
      markers) and median-of-means estimator (in~$\bullet$ markers)
      averaged over $T = 1000$ trials with $\Nsamp = 10^4$ samples
      each. We have also plotted the least-squares fits through these
      points, and the slopes of these lines are provided in the
      legend. In particular, the legend contains the tuple of slopes
      $(\betaplug, \betamom, \beta^*)$ for each value of
    $\alpha$. Logarithms are to the natural base.}
    \label{fig:mrp}
  \end{center}
\end{figure}

In order to parameterize this MRP in a scalarized manner, we vary the
triple $(p, \nu, \tau)$ in the following way.  First, we fix a scalar
$\alpha$ in the unit interval $[0, 1]$, and then we set
 \begin{align*}
 p = \tfrac{4 \discount - 1}{3 \discount}, \qquad \lbmaxrew = 1, \quad
 \text{ and } \quad \; \tau = 1 - (1 - \discount)^{\alpha}.
 \end{align*}
Note that this sub-family of MRPs is fully parameterized by the pair
$(\discount, \alpha)$.  Let us clarify why this particular
scalarization is interesting.  As shown in the proof of
Theorem~\ref{thm:lb} (see equation~\eqref{eq:basic-mrp-props}), the
underlying MRP has maximal standard deviation scaling as
\begin{align*}
 \| \sigma(\thetastar) \|_{\infty} \sim \left( \frac{1}{1 - \discount}
 \right)^{0.5 - \alpha}.
\end{align*}
Consequently, by the bound~\eqref{eq:robust-result} from
Theorem~\ref{thm:rob}, for a fixed sample size $\Nsamp$, the MoM
estimator should have $\ell_\infty$-norm scaling as $\left( \frac{1}{1
  - \discount} \right)^{1.5 - \alpha}$.  As we discuss in
Appendix~\ref{AppAux}, the same prediction also holds for the plug-in
estimator, assuming that $\Nsamp \succsim \frac{1}{(1-\discount)}$.

In order to test this prediction, we fixed the parameter $\alpha \in
[0,1]$, and generated a range of MRPs with different values of the
discount factor $\discount$.  For each such MRP, we drew $\Nsamp =
10^4$ samples from the generative observation model and computed both
the plug-in and median-of-means estimators, where the latter estimator
was run with the choice $K = 20$.  While the plug-in estimator has a
simple closed-form expression, the MoM estimator was obtained by
running the median-of-means Bellman operator
$\Belemp_{\Nsamp}^{\robust}$ iteratively until it converged to its
fixed point; we declared that convergence had occurred when the
$\ell_\infty$-norm of the difference between successive iterates fell
below $10^{-8}$.

In panel (b) of Figure~\ref{fig:mrp}, we plot the $\ell_\infty$-error,
of both the plug-in approach as well as the median-of-means estimator,
as a function of $\discount$.  The plot shows the behavior for three
distinct values $\alpha = \{0, 0.5, 1\}$.  Each point on each curve is
obtained by averaging $1000$ Monte Carlo trials of the experiment.
Note that on this log-log plot, we see a linear relationship between
the log $\ell_\infty$-error and log discount complexity, with the
slopes depending on the value of $\alpha$.  More precisely, from our
calculations above, our theory predicts that the log
$\ell_\infty$-error should be related to the log complexity $\log
\big( \tfrac{1}{1 - \discount} \big)$ in a linear fashion with slope
\begin{align*}
  \betastar & = 1.5 - \alpha.
\end{align*}
Consequently, for both the plug-in and MoM estimators, we performed a
linear regression to estimate these slopes, denoted by $\betaplug$ and
$\betamom$ respectively.  The plot legend reports the triple
$(\betaplug, \betamom, \betastar)$, and for each we see good agreement
between the theoretical prediction $\betastar$ and its empirical counterparts.

\subsection{When does the MoM estimator perform better than plug-in?}

Our theoretical results predict that the MoM estimator should outperform
the plug-in approach when the span semi-norm of the value function $\| \thetastar \|_{\myspan}$ is much larger than its maximum
standard deviation $\| \sigma(\thetastar) \|_{\infty}$. Indeed,
Proposition~\ref{prop:plugin-lb} in Appendix~\ref{AppPluginLB} demonstrates that there are MRPs on which the $\ell_\infty$-error of the plug-in estimator grows with the 
span semi-norm of the optimal value function. Let us now simulate
the behavior of both the plug-in and MoM approach on this MRP, constructed
by taking $\Dim / 3$ copies of the $3$-state MRP in Figure~\ref{fig:mrp2}(a).

\begin{figure}[ht]
  \begin{center}
    \begin{tabular}{ccc}
          \raisebox{1in}{\widgraph{0.35\textwidth}{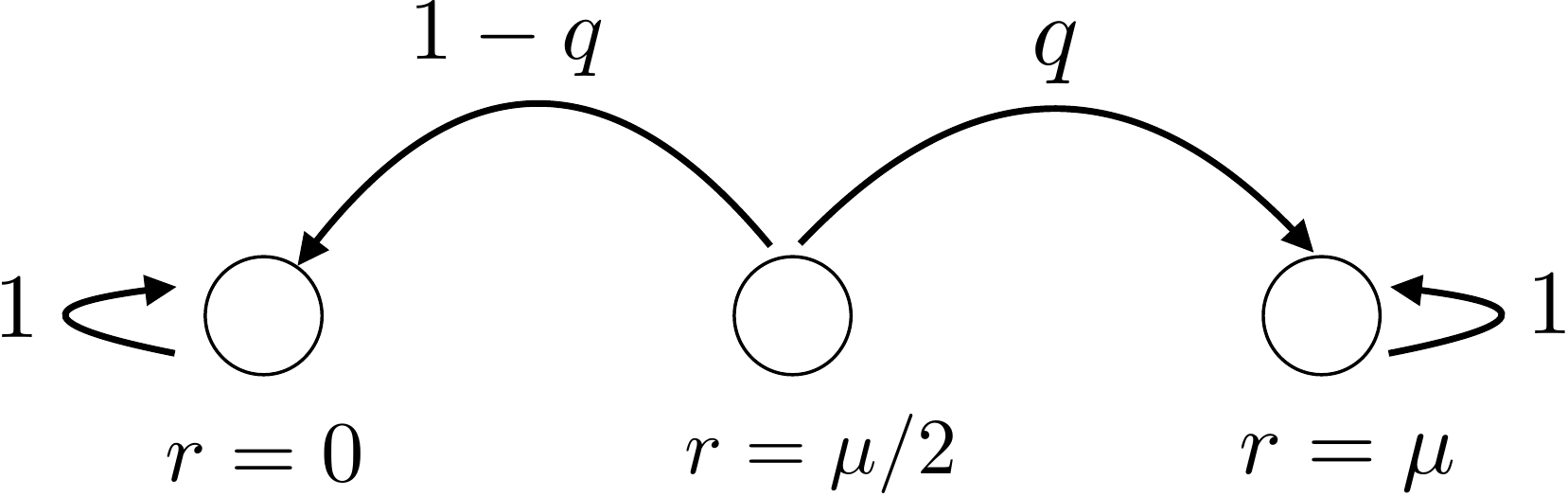}} &&
          \widgraph{0.58\textwidth}{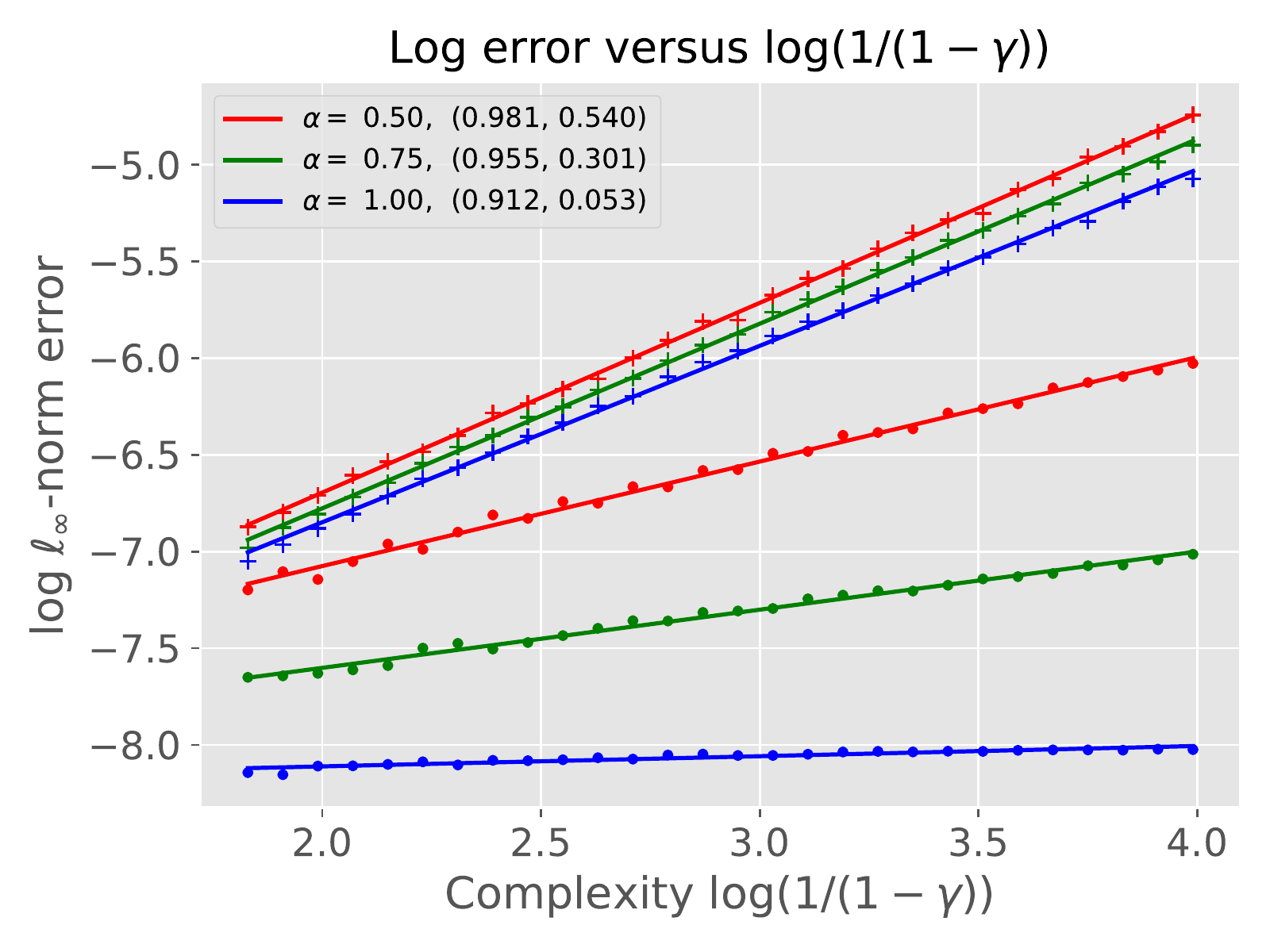}
          \\ (a) & & (b) \\
          && 
    \end{tabular}
    \caption{(a) Illustration of the MRP $\MRP_{\secondbasic} (q, \mu)$. 
    In the simulation as well as in
      the lower bound construction of Proposition~\ref{prop:plugin-lb},
      we concatenate $\Dim / 3$ such MRPs to produce an
      MRP on $\Dim$ states. For the simulation, we choose $\Dim = 3 \Big\lfloor \left( \frac{1}{1 - \discount} \right)^{\alpha} \Big\rfloor$ and set $q = \frac{10}{\Nsamp \Dim}$ and $\mu = 1$.  
      (b) Log-log plot of the
      $\ell_\infty$-error versus the discount complexity parameter $1
      / (1 - \discount)$ for both the plug-in estimator (in $+$
      markers) and median-of-means estimator (in~$\bullet$ markers)
      averaged over $T = 1000$ trials with $\Nsamp = 10^4$ samples
      each. We have also plotted the least-squares fits through these
      points, and the slopes of these lines are provided in the
      legend. In particular, the legend contains the tuple of slopes
      $(\betaplug, \betamom)$ for each value of
    $\alpha$. Logarithms are to the natural base.}
    \label{fig:mrp2}
  \end{center}
\end{figure}

Our simulation is carried out on $\Nsamp = 10^4$ samples from this
$\Dim$-state MRP, with noiseless observations of the reward.  In order
to parameterize the MRP via the discount factor alone, we fix the pair
$(q, \Dim)$ in the following way.  First, we fix a scalar $\alpha$ in
the unit interval $[0, 1]$, and then set
 \begin{align*}
 \Dim = 3 \Big\lfloor \left( \frac{1}{1 - \discount} \right)^{\alpha}
 \Big\rfloor \quad \text{ and } \quad \; q = \frac{10}{\Nsamp \Dim}.
 \end{align*}
Note that this sub-family of MRPs is fully parameterized by the pair
$(\discount, \alpha)$. The construction also ensures that
\begin{align}
\label{eq:span-gtr-var}
\frac{\| \thetastar \|_{\infty}}{\Nsamp} \gg \frac{\|
  \sigma(\ValueFunc^*) \|_{\infty}}{\sqrt{\Nsamp}},
\end{align}
for this chosen parameterization, and furthermore, that the ratio
 of the LHS and RHS of inequality~\eqref{eq:span-gtr-var} increases
 as the dimension $\Dim$ increases (see the proof of Proposition~\ref{prop:plugin-lb}).

As shown in Proposition~\ref{prop:plugin-lb} in
Appendix~\ref{AppPluginLB}, the $\ell_\infty$ error of the plug-in
estimator for this family of MRPs can be \emph{lower bounded} by $\|
\thetastar \|_{\infty} / \Nsamp$. It is also straightforward to show
that the error of the MoM estimator is \emph{upper bounded} by the
quantity $\frac{\| \sigma(\ValueFunc^*)
  \|_{\infty}}{\sqrt{\Nsamp}}$. Now increasing the value of $\alpha$
increases the dimension $\Dim$, and so the MoM estimator should behave
better and better for larger values of $\alpha$. In particular, this
behavior can be captured in the log-log plot of the error against $1 /
(1 - \discount)$, which is presented in Figure~\ref{fig:mrp2}(b).



The plot shows the behavior for three
distinct values $\alpha = \{0.5, 0.75, 1\}$.  Each point on each curve is
obtained by averaging $1000$ Monte Carlo trials of the experiment.
As expected, the MoM estimator consistently outperforms the 
plug-in estimator for each value of $\alpha$. Moreover, on this log-log plot, we see a linear relationship between
the log $\ell_\infty$-error and log discount complexity, with the
slopes depending on the value of $\alpha$. For both the plug-in and MoM estimators, we performed a linear regression to estimate these slopes, denoted by $\betaplug$ and
$\betamom$ respectively.  The plot legend reports the pair
$(\betaplug, \betamom)$, and we see that the gap between the slopes
increases as $\alpha$ increases.


\section{Proofs}
\label{sec:proofs}

We now turn to the proofs of our main results.  Throughout our proofs,
the reader should recall that the values of absolute constants may
change from line-to-line.  We also use the following facts
repeatedly. First, for a row stochastic matrix~$\Mmat$ with
non-negative entries and any scalar $\discount \in [0,1)$, we have the
  infinite series
\begin{subequations}
  \begin{align}
    \label{EqnInfinite}
(\IdMat - \discount \Mmat)^{-1} = \sum_{t = 0}^{\infty} (\discount
    \Mmat)^t,
\end{align}
which implies that the entries of $(\IdMat - \gamma \Mmat)^{-1}$ are
all non-negative. Second, for any such matrix, we also have the bound
$\| (\IdMat - \gamma \Mmat)^{-1} \|_{1, \infty} \leq \frac{1}{1 -
  \discount}$.  Finally, for any matrix $\Amat$ with positive entries
and a vector $v$ of compatible dimension, we have the elementwise
inequality
\begin{align}
 \label{EqnHandy}
|\Amat v| \preceq \Amat |v |.
\end{align}
\end{subequations}


\subsection{Proof of Theorem~\ref{thm:plugin}, part (a)}

Throughout this proof, we adopt the convenient shorthand $\thetahat
\equiv \ValueFunchat$ for notational convenience.  By the Bellman
equations~\eqref{eq:Bellman-true} and~\eqref{eq:Bellman-est} for
$\ValueFunc^*$ and $\thetahat$, respectively, we have
\begin{align*}
\thetahat - \thetastar & = \discount \left \{ \Phat \thetahat - \Pmat
\ValueFunc^* \right \} + (\rhat - \reward) \; = \; \discount \Phat
(\thetahat - \ValueFunc^*) + \discount (\Phat - \Pmat) \ValueFunc^* +
(\rhat - \reward).
\end{align*}
Introducing the shorthand $\DelHat \defn \thetahat - \thetastar$ and
re-arranging implies the relation
\begin{align}
\label{eq:error-relation-bellman}
\DelHat & = \discount (\IdMat - \discount \Phat)^{-1} (\Phat - \Pmat)
\thetastar + (\IdMat - \discount \Phat)^{-1} (\rhat - \reward),
\end{align}
and consequently, the elementwise inequality
\begin{align}
  \label{EqnNewError}
|\DelHat| & \preceq \discount (\IdMat - \discount \Phat)^{-1} |(\Phat
- \Pmat) \thetastar| + (\IdMat - \discount \Phat)^{-1} |(\rhat -
\reward)|,
\end{align}
where we have used the relation~\eqref{EqnHandy} with the matrix
$\Amat = (\IdMat - \discount \Phat)^{-1}$.  Given the sub-Gaussian
condition on the stochastic rewards, we can apply Hoeffding's
inequality combined with the union bound to obtain the elementwise
inequality $|\rhat - \reward| \preceq \unicon \LOGDN \cdot \sigmar$,
which holds with probability at least $1 - \tfrac{\delta}{4}$.  Since
the matrix $(\IdMat - \discount \Phat)^{-1}$ has non-negative entries
and $(1, \infty)$-norm at most $\tfrac{1}{1-\discount}$, we have
\begin{subequations}
\label{EqnTailBounds}
\begin{align}
  \label{eq:Hoeffding-rew}  
  (\IdMat - \discount \Phat)^{-1} |\rhat - \reward| & \preceq
  \frac{\unicon}{1-\discount} \| \sigmar \|_{\infty} \LOGDN \ones.
\end{align}
with the same probability.  On the other hand, by Bernstein's
inequality, we have
\begin{align*}
  |(\Phat - \Pmat) \thetastar | \preceq \unicon \Bigg \{ \LOGDN \;
  \cdot \; \sigma(\thetastar) + \|\thetastar\|_{\myspan} \LOGDNSQ
  \cdot \ones \Bigg \}
\end{align*}
with probability at least $1- \tfrac{\delta}{4}$, and hence
\begin{align}
  \label{EqnBasicBern}  
(\IdMat - \discount \Phat)^{-1} |(\Phat - \Pmat) \thetastar | &
  \preceq \unicon \Bigg \{ \LOGDN \; \cdot \;\| (\IdMat - \discount
  \Phat)^{-1} \sigma(\thetastar)\|_\infty +
  \frac{\|\thetastar\|_{\myspan}}{1-\discount} \LOGDNSQ
  \Bigg \} \cdot \ones.
\end{align}
\end{subequations}
Substituting the bounds~\eqref{eq:Hoeffding-rew}
and~\eqref{EqnBasicBern} into the elementwise
inequality~\eqref{EqnNewError}, we find that
\begin{align}
  \label{EqnInterErr}
  |\DelHat| & \preceq \unicon \Bigg \{ \LOGDN \; \cdot \; \left(
  \discount\| (\IdMat - \discount \Phat)^{-1}
  \sigma(\thetastar)\|_\infty + \frac{\|\sigmar\|_\infty}{1 -
    \discount} \right) + \frac{\discount
    \|\thetastar\|_{\myspan}}{1-\discount} \LOGDNSQ \Bigg
  \} \cdot \ones
\end{align}
with probability at least $1 - \tfrac{\delta}{2}$.

Our next step is to relate the pair of population quantities
$(\sigma(\thetastar), \| \thetastar \|_{\myspan})$ to their empirical
analogues~$(\sighat(\thetahat), \| \thetahat \|_{\myspan})$. The
following lemma provides such a bound.
\begin{lemma}[Population to empirical variance]
\label{LemSigToHat}
We
have the element-wise inequality
\begin{align}
\label{EqnSigToHat}
\sigma(\thetastar) & \preceq 2 \sighat(\thetahat) + 2 |\DelHat| +
\unicon' \|\thetastar\|_{\myspan} \LOGDN \cdot \ones
\end{align}
with probability at least $1-\delta/2$.
\end{lemma}
\noindent Taking this lemma as given for the moment, let us complete
the proof. \\

Since the matrix $(\IdMat - \discount \Phat)^{-1}$ has non-negative
entries, we can multiply both sides of the elementwise
inequality~\eqref{EqnSigToHat} by it; doing so and taking the
$\ell_\infty$-norm yields
\begin{align*}
  \|(\IdMat - \discount \Phat)^{-1} \sigma(\thetastar)\|_\infty & \leq
  2\|(\IdMat - \discount \Phat)^{-1} \sighat(\thetahat)\|_\infty +
  \frac{2\|\DelHat\|_\infty}{1- \discount} + \frac{\unicon'
    \spannorm{\thetastar}}{1-\discount} \LOGDN.
\end{align*}
Substituting back into the elementwise inequality~\eqref{EqnInterErr}
and taking $\ell_\infty$-norms of both sides, we find that
\begin{multline*}
\|\DelHat\|_\infty \leq \unicon \Bigg \{ \LOGDN \left( \discount
\|(\IdMat - \discount \Phat)^{-1} \sighat(\thetahat) \|_\infty +
\frac{\|\sigmar\|_\infty}{1-\discount} \right) + \frac{\discount \|
  \thetastar\|_{\myspan}}{1-\discount} \LOGDNSQ \Bigg \} \\ + \frac{2
  \unicon \discount}{1-\discount} \LOGDN \|\DelHat\|_\infty.
\end{multline*}
Since the span semi-norm satisfies the triangle inequality, we have
\begin{align*}
\|\thetastar\|_{\myspan} \leq \|\thetahat\|_{\myspan} +
\|\DelHat\|_{\myspan} \leq \|\thetahat\|_{\myspan} + 2
\|\DelHat\|_{\infty}.
\end{align*} 
Substituting this bound and re-arranging yields
\begin{align*}
  \kappa \| \thetahat - \thetastar \|_{\infty} & \leq \unicon \Bigg \{
  \LOGDN \left( \discount \|(\IdMat - \discount \Phat)^{-1}
  \sighat(\thetahat) \|_\infty +
  \frac{\|\sigmar\|_\infty}{1-\discount} \right) + \frac{\discount \|
    \thetahat\|_{\myspan}}{1-\discount} \LOGDNSQ \Bigg \}.
\end{align*}
where we have introduced the shorthand $\kappa \defn 1 - \frac{ 2
  \unicon \discount}{1-\discount} \left( \LOGDN + \LOGDNSQ \right)$.
Finally, by choosing the pre-factor $\unicon_1$ in the lower bound
$\Nsamp \geq \unicon_1 \discount^2 \frac{\log(8 \Dim/\delta)}{(1 -
  \discount)^2}$ large enough, we can ensure that $\kappa \geq
\frac{1}{2}$, thereby completing the proof of
Theorem~\ref{thm:plugin}(a).

\subsubsection{Proof of Lemma~\ref{LemSigToHat}}

We now turn to the proof of the auxiliary result in
Lemma~\ref{LemSigToHat}. We begin by noting that the statement
is trivially true when $\Nsamp \leq \log (8 \Dim / \delta)$,
since we have
\begin{align*}
\sigma (\theta^*) \preceq \| \thetastar \|_{\myspan} \ones.
\end{align*}
Thus, by adjusting the constant factors in the statement of the lemma,
it suffices to prove the lemma under the assumption \mbox{$\Nsamp \geq c \log (8 \Dim / \delta)$} for
a sufficiently large absolute constant $c$. Accordingly, we make this assumption
for the rest of the proof.

We use the following convenient notation for expectations. Let $\Exs$
denote the vector expectation operator, with the convention that $\Exs
[v] = \Pmat v$. Similarly, let $\Ehat$ denote the vector empirical
expectation operator, given by $\Ehat [v] = \Phat v$. These operators
are applied elementwise by definition, and we let $\Exs_i$ and
$\Ehat_i$ denote the $i$-th entry of each operator, respectively.

With this notation, we have
\begin{align}
\sigma^2(\thetastar) & = \Exs \left|\thetastar - \Exs[\thetastar]
\right|^2 \notag \\
& = (\Exs - \Ehat) \Big|\thetastar - \Exs[\thetastar] \Big|^2  + \Ehat
  \left|\thetastar - \Exs[\thetastar] \right|^2  \notag \\
& \preceq (\Exs - \Ehat) \Big|\thetastar - \Exs[\thetastar] \Big|^2 +
  2 \left|\Ehat[\thetastar] - \Exs[\thetastar] \right|^2 + 2 \Ehat
  \left| \thetastar - \Ehat[\thetastar] \right|^2 \notag \\
  \label{EqnSpill}
& = \underbrace{(\Exs - \Ehat) \Big|\thetastar - \Exs[\thetastar]
    \Big|^2}_{\Term_1} +  \underbrace{2 \left|\Ehat[\thetastar] -
    \Exs[\thetastar] \right|^2}_{\Term_2} + 2 \sighat^2(\thetastar).
\end{align}

We claim that the terms $\Term_1$ and $\Term_2$ are bounded as follows:
\begin{subequations}
\begin{align}
\label{EqnTerm1}
\Term_1 & \preceq \frac{\sigma^2(\thetastar)}{4} + \unicon
\|\thetastar\|_{\myspan}^2 \LOGDNSQ \cdot \ones, \quad \mbox{and} \\
\label{EqnTerm2}
\Term_2 & \preceq \unicon \left \{ \LOGDNSQ \cdot \sigma^2(\thetastar) +
\left(\|\thetastar\|_{\myspan} \LOGDNSQ \right)^2 \cdot \ones \right \},
\end{align}
\end{subequations}
where each bound holds with probability at least $1 -
\frac{\delta}{4}$.  Taking these bounds as given for the moment,
as long as $\Nsamp \geq \unicon' \log(8 \Dim/\delta)$
for a sufficiently large constant $\unicon'$, we can ensure that
\begin{align*}
\Term_1 + \Term_2 & \preceq \frac{\sigma^2(\thetastar)}{2} + \unicon
\|\thetastar\|_{\myspan}^2 \LOGDNSQ \cdot \ones,
\end{align*}
Substituting back into our earlier bound~\eqref{EqnSpill}, we find that
\begin{align*}
\frac{\sigma^2(\thetastar)}{2} & \preceq 2 \sighat^2(\thetastar) +
\unicon' \|\thetastar\|_{\myspan}^2 \LOGDNSQ \cdot \ones.
\end{align*}
Rearranging and taking square roots entry-wise, we find that
\begin{align*}
\sigma(\thetastar) & \preceq \sqrt{4 \sighat^2(\thetastar) + 2
  \unicon' \|\thetastar\|_{\myspan}^2 \LOGDNSQ \cdot \ones} \; \preceq \; 2
\sighat(\thetastar) + \unicon' \|\thetastar\|_{\myspan} \LOGDN \cdot \ones.
\end{align*}
Finally noting that we have the entry-wise inequality
$\sighat(\thetastar) \preceq \sighat(\thetahat) + | \thetahat - \ValueFunc^* |$ establishes the claim of Lemma~\ref{LemSigToHat}.

\noindent It remains to prove the bounds~\eqref{EqnTerm1}
and~\eqref{EqnTerm2}.

\paragraph{Proof of bound~\eqref{EqnTerm1}:}

For each index $i \in [\Dim]$, define the random variable $Y_i \defn
\big( \thetastar_{J} - \Exs_i[\thetastar] \big)^2$, where $J$ is an
index chosen at random from the distribution $\pmat_i$. By definition,
each random variable $Y_i$ is non-negative, and so with $\EE$ now denoting the regular expectation of a scalar random variable, we have lower tail
bound (Proposition 2.14,~\cite{Wai19})
\begin{align*}
\mprob \left[ \Exs [Y_i] - Y_i \geq s \right] & \leq \exp \left(
- \frac{\numobs s^2}{2 \Exs[Y_i^2]} \right) \qquad \mbox{for all $s
  > 0$.}
\end{align*}
Moreover, we have $Y_i \leq \| \thetastar \|^2_{\myspan}$ almost
surely, from which we obtain
\begin{align*}
\Exs[Y_i^2] & \leq \|\thetastar\|^2_{\myspan} \Exs_i
\left[(\thetastar - \Exs[\thetastar])^2 \right] \; = \;
\|\thetastar\|_{\myspan}^2 \sigma_i^2(\thetastar).
\end{align*}
Putting together the pieces yields the elementwise inequality
\begin{align*}
\Term_1 & \preceq \unicon
\|\thetastar\|_{\myspan} \LOGDN \cdot
\sigma(\thetastar) \stackrel{\1}{\preceq} \;
\frac{\sigma^2(\thetastar)}{8} + \unicon'
\|\thetastar\|_{\myspan}^2\LOGDNSQ,
\end{align*}
with probability at least $1 - \delta/3$, where in step $\1$, we have
used the inequality $2ab \leq \lbmaxrew a^2 + \lbmaxrew^{-1} b^2$,
which holds for any triple of positive scalars $(a, b, \lbmaxrew)$.

\paragraph{Proof of the bound~\eqref{EqnTerm2}:}

From Bernstein's inequality, we have the element-wise bound
\begin{align*}
\Big |\Ehat[\thetastar] - \Exs[\thetastar] \Big| & \preceq \unicon
\left \{ \LOGDN \cdot \sigma(\thetastar) + \|\thetastar\|_{\myspan}
\LOGDNSQ \cdot \ones \right \}
\end{align*}
with probability at least $1-\delta/4$, and hence
\begin{align*}
\Term_2 & \preceq \unicon \left \{  \LOGDNSQ \cdot \sigma^2(\thetastar) +
\left(\|\thetastar\|_{\myspan} \LOGDNSQ \right)^2 \cdot \ones \right \},
\end{align*}
as claimed.


\subsection{Proof of Theorem~\ref{thm:plugin}, part (b)}
\label{sec:pf-thm1b}

Once again, we employ the shorthand $\thetahat \equiv \ValueFunchat$
for notational convenience, and also the shorthand~$\DelHat =
\thetahat - \thetastar$.  Note that it suffices to show the inequality
\begin{align}
\label{eq:equiv-bd}
\Prob \left\{ \|\thetahat - \thetastar\|_\infty \geq \unicon \discount
\left \| (\IdMat - \discount \Pmat)^{-1} |(\Phat - \Pmat) \thetastar|
\right\|_\infty + c (1 - \discount)^{-1} \| \rhat - \reward \|_\infty
\right\} &\leq \frac{\delta}{2},
\end{align}
from which the theorem follows by application of a Bernstein bound to
the first term and Hoeffding bound to the second, in a similar fashion
to the inequalities~\eqref{EqnTailBounds}. We therefore dedicate the
rest of the proof to establishing inequality~\eqref{eq:equiv-bd}.

\subsubsection{Proving the bound~\eqref{eq:equiv-bd}}

We have
\begin{align*}
\DelHat = \thetahat - \thetastar & = \discount \Phat \thetahat -
\discount \Pmat \thetastar + (\rhat - \reward) = \discount (\Phat -
\Pmat) \thetahat + \discount \Pmat \DelHat + (\rhat - \reward),
\end{align*}
which implies that
\begin{align}
\label{EqnBaseRelation}
\DelHat - (\IdMat - \discount \Pmat)^{-1} (\rhat - \reward) =
\discount (\IdMat - \discount \Pmat)^{-1} (\Phat - \Pmat) \thetahat \;
= \; \discount (\IdMat - \discount \Pmat)^{-1} (\Phat - \Pmat) \DelHat
+ \discount (\IdMat - \discount \Pmat)^{-1} (\Phat - \Pmat)
\thetastar.
\end{align}



Since all entries of $(\Id - \discount \Pmat)^{-1}$ are non-negative,
we have the element-wise inequalities
\begin{align}
|\DelHat| & \preceq \discount (\IdMat - \discount \Pmat)^{-1} |(\Phat
- \Pmat) \DelHat| + \discount (\IdMat - \discount \Pmat)^{-1} |(\Phat
- \Pmat) \thetastar| + (\IdMat - \discount \Pmat)^{-1} |\rhat -
\reward| \notag \\ &\preceq \discount (\IdMat - \discount \Pmat)^{-1}
|(\Phat - \Pmat) \DelHat| + \discount (\IdMat - \discount \Pmat)^{-1}
|(\Phat - \Pmat) \thetastar| + \frac{1}{1 - \discount} \| \rhat -
\reward \|_\infty \cdot \ones.  \label{EqnTriangleUpper}
\end{align}
The second and third terms are already in terms of the desired
population-level functionals in equation~\eqref{eq:equiv-bd}.  It
remains to bound the first term.

Note that the key difficulty here is the fact that the two matrices
$\Phat - \Pmat$ and $\DelHat$ are not independent.  As a first attempt
to address this dependence, one is tempted to use the fact that
provided $\Nsamp$ is large enough, each row of $\Phat - \Pmat$ has
small $\ell_1$-norm; for instance, see Weissman et
al.~\cite{weissman2003inequalities} for sharp bounds of this type. In
particular, this would allow us to work with the entry-wise bounds
\begin{align*}
|(\Phat - \Pmat) \DelHat| \preceq \| \Phat - \Pmat \|_{1, \infty} \|
\DelHat \|_\infty \cdot \ones \precsim C \sqrt{\frac{\Dim}{N}} \|
\DelHat \|_\infty \cdot \ones,
\end{align*}
where the final relation hides logarithmic factors in the pair $(\Dim,
\delta)$.  Proceeding in this fashion, we would then bound each entry
in the first term of equation~\eqref{EqnTriangleUpper} by $\discount
(1 - \discount)^{-1} \sqrt{\frac{\Dim}{\Nsamp}} \| \DelHat \|_\infty$;
then choosing $N$ large enough such that $\discount (1 -
\discount)^{-1} \sqrt{\frac{\Dim}{\Nsamp}} \leq 1/2$ suffices to
establish bound~\eqref{eq:equiv-bd}. However, this requires a sample
size $\Nsamp \gtrsim \frac{\discount^2}{(1 - \discount)^2} \Dim$,
while we wish to obtain the bound~\eqref{eq:equiv-bd} with the sample
size $\Nsamp \gtrsim \frac{\discount^2}{(1 - \discount)^2}$. This
requires a more delicate analysis.

Our analysis instead proceeds entry-by-entry, and uses a leave-one-out
sequence to carefully decouple the dependence between $\Phat - \Pmat$
and $\DelHat$. Let us introduce some notation to make this
precise. For each $i \in [\Dim]$, recall that we used $\phat_i$ and
$\pmat_i$ to denote row $i$ of the matrices $\Phat$ and $\Pmat$,
respectively. Let $\Phatloo{i}$ denote the $i$-th leave-one-out
transition matrix, which is identical to $\Phat$ except with row $i$
replaced by the population vector $\pmat_i$. Let $\thetahatIpert \defn
(\IdMat - \discount \Phatloo{i})^{-1} r$ be the value function
estimate based on $\Phatloo{i}$ and the \emph{true} reward vector
$\reward$, and denote the associated difference vector by
$\DelHatIpert \defn \thetahatIpert - \thetastar$.

Now note that we have
\begin{align*}
\left[ (\Phat - \Pmat) \DelHat \right]_i = \inprod{\phat_i -
  \pmat_i}{\DelHat} = \inprod{\phat_i - \pmat_i}{\DelHatIpert} +
\inprod{\phat_i - \pmat_i}{\thetahat - \thetahatIpert}.
\end{align*}
This decomposition is helpful because, now, the vectors $\phat_i -
\pmat_i$ and $\DelHatIpert$ are independent by construction, so that
standard tail bounds can be used on the first term. For the second
term, we use the fact that $\thetahat \approx \thetahatIpert$, since
the latter is obtained by replacing just one row of the estimated
transition matrix.  Formally, this closeness will be argued by using
the matrix inversion formula. We collect these two results in the
following lemma.
\begin{lemma}
\label{LemDoubleBound}
Suppose that the sample size is lower bounded as $\Nsamp \geq \unicon'
\discount^2 \frac{\log(8 \Dim/\delta)}{(1-\discount)^2}$. Then with
probability at least $1 - \frac{\delta}{2 \Dim}$ and for each $i \in
[\Dim]$, we have
\begin{subequations}  
  \begin{align}
 \label{EqnDoubleBoundA}
\discount |\inprod{\phat_i - \pmat_i}{\DelHatIpert}| & \leq c \left \{ \discount
\|\DelHat\|_\infty \sqrt{ \frac{\log(8 \Dim/\delta)}{\Nsamp}} + \discount \left
|\inprod{\phat_i - p_i}{\thetastar} \right | + \|
\reward - \rhat \|_\infty \right \} \quad \mbox{and} \\
 \label{EqnDoubleBoundB}
\discount |\inprod{\phat_i - \pmat_i}{\thetahatIpert - \thetahat}| & \leq c
\left \{ \discount \|\DelHat\|_\infty \sqrt{ \frac{\log(8 \Dim/\delta)}{\Nsamp}}
+ \discount \left |\inprod{\phat_i - p_i}{\thetastar} \right | +
\| \reward - \rhat \|_\infty\right \}.
\end{align}
\end{subequations}
\end{lemma}

With this lemma in hand, let us complete the proof.  Combining the
bounds of Lemma~\ref{LemDoubleBound} with a union bound over all
$\Dim$ entries yields the elementwise inequality
\begin{align*}
\discount \left |(\Phat - \Pmat) \DelHat \right| & \preceq c \discount \left| (\Phat -
\Pmat) \thetastar \right | + \unicon \left \{ \discount \|\DelHat\|_\infty
\sqrt{ \frac{\log(8 \Dim/\delta)}{\Nsamp}} + \| \rhat -
\reward \|_\infty \right\} \onevec
\end{align*}
with probability at least $1 - \delta / 2$.  Since the entries of
$(\IdMat - \discount \Pmat)^{-1}$ are non-negative, we can multiply
both sides of this inequality by it, thereby obtaining
\begin{align*}
\discount (\IdMat - \discount \Pmat)^{-1} \left |(\Phat - \Pmat) \DelHat \right|
  & \preceq c \discount (\IdMat - \discount \Pmat)^{-1} \left| (\Phat - \Pmat)
  \thetastar \right | + \frac{c}{1 - \discount} \left \{
 \discount  \|\DelHat\|_\infty \sqrt{ \frac{\log(8 \Dim/\delta)}{\Nsamp}} +
  \| \rhat - \reward \|_\infty \right \} \ones.
\end{align*}
Returning to the upper bound~\eqref{EqnTriangleUpper}, we have shown that
\begin{align*}
\|\DelHat\|_\infty & \leq c \discount
\frac{\|\DelHat\|_\infty}{1-\discount} \sqrt{ \frac{\log(8
    \Dim/\delta)}{\Nsamp}} + \unicon' \discount \left \| (\IdMat -
\discount \Pmat)^{-1} |(\Phat - \Pmat) \thetastar| \right\|_\infty +
\frac{c}{1 - \discount} \| \reward - \rhat \|_\infty.
\end{align*}
Under the assumed lower bound on the sample size $\Nsamp \geq \unicon'
\discount^2 \frac{\log(8 \Dim/\delta)}{(1-\discount)^2}$, this
inequality implies that
\begin{align*}
\|\DelHat\|_\infty & \leq \unicon' \discount \left \| (\IdMat - \discount
  \Pmat)^{-1}  |(\Phat - \Pmat) \thetastar| \right\|_\infty + \frac{c}{1 - \discount} \| \reward - \rhat \|_\infty,
\end{align*}
as claimed~\eqref{eq:equiv-bd}. \qed

We now proceed to a proof of Lemma~\ref{LemDoubleBound}, which uses
the following structural lemma relating the quantities $\DelHatIpert$
and $\DelHat$.

\begin{lemma}
  \label{LemBoundingIpert}
Suppose that the sample size is lower bounded as $\Nsamp \geq \unicon'
\discount^2 \frac{\log(8 \Dim/\delta)}{(1-\discount)^2}$. Then with
probability at least $1 - \frac{\delta}{4\Dim}$ and for each $i \in
[\Dim]$, we have
\begin{align}
  \label{EqnBoundingIpert}
\|\DelHatIpert\|_\infty & \leq c \|\DelHat\|_\infty + \frac{c}{1-\discount} \Big\{ \discount \left |\inprod{\phat_i -
  p_i}{\thetastar} \right | + \| \rhat - \reward
\|_\infty \Big\}.
\end{align}
\end{lemma}
\noindent This lemma is proved in Section~\ref{sec:BoundingIpert} to
follow.


\subsubsection{Proof of Lemma~\ref{LemDoubleBound}}

We prove the two bounds in turn.

\paragraph{Proof of inequality~\eqref{EqnDoubleBoundA}:}
Note that $\phat_i - \pmat_i$ and $\DelHatIpert$ are independent by
construction, so that the Hoeffding inequality yields
\begin{align}
  \label{EqnAshwinOne}
|\inprod{\phat_i - \pmat_i}{\DelHatIpert}| & \leq c \|\DelHatIpert
\|_\infty \sqrt{ \frac{\log(8 \Dim/\delta)}{\Nsamp}}
\end{align}
with probability at least $1 - \delta/(4\Dim)$.

Using this in conjunction with inequality~\eqref{EqnBoundingIpert} from
Lemma~\ref{LemBoundingIpert}
yields the bound
\begin{align*}
 \discount |\inprod{\phat_i - \pmat_i}{\DelHatIpert}| & \leq c \discount
  \|\DelHat\|_\infty \sqrt{ \frac{\log(8 \Dim/\delta)}{\Nsamp}} +
  \frac{c \discount}{1-\discount} \LOGDN \Big\{ \discount \left
  |\inprod{\phat_i - p_i}{\thetastar} \right | + \|
  \rhat - \reward \|_\infty \Big\} \\
& \stackrel{\1}{\leq} c \discount \|\DelHat\|_\infty
 \sqrt{ \frac{\log(8 \Dim/\delta)}{\Nsamp}} + 
c \discount \left |\inprod{\phat_i - p_i}{\thetastar} \right | + \unicon \| \rhat - \reward \|_\infty,
\end{align*}
where in step $\1$, we have used the lower bound on the sample size
$\Nsamp \geq \unicon' \frac{\discount^2}{(1 - \discount)^2} \log
(8\Dim / \delta)$.

\paragraph{Proof of inequality~\eqref{EqnDoubleBoundB}:}

The proof of this claim is more involved.  Using the
relation~\eqref{EqnBaseRelation} (with suitable modifications of
terms), we have
\begin{align}
  \label{EqnExplicitIpert}
\thetahatIpert - \thetahat &= \discount (\IdMat - \discount
\Phat)^{-1} (\Phatloo{i} - \Phat) \thetahatIpert + (\IdMat - \discount
\Phat)^{-1} (\reward - \rhat) \notag \\ &= - \discount (\IdMat -
\discount \Phat)^{-1} e_i \left(\inprod{\phat_i - p_i}{\thetahatIpert}
\right) + (\IdMat - \discount \Phat)^{-1} (\reward - \rhat).
\end{align}
Moreover, the Woodbury matrix identity~\cite{Horn85} yields
\begin{align*}
\Mmat \defn \Big( \IdMat - \discount \Phat \Big)^{-1} - \Big(\IdMat -
\discount \Phatloo{i} \Big)^{-1} & = - \gamma \frac{(\IdMat -
  \discount \Phatloo{i})^{-1} e_i (\phat_i - p_i)^T (\IdMat -
  \discount \Phatloo{i})^{-1}}{1 - \gamma (\phat_i - p_i)^T (\IdMat -
  \discount \Phatloo{i})^{-1} e_i}.
\end{align*}
Consequently,
\begin{align}
\inprod{\phat_i - \pmat_i}{\thetahatIpert - \thetahat} &= - \discount
(\phat_i - \pmat_i)^\top (\IdMat - \discount \Phat)^{-1} e_i
\left(\inprod{\phat_i - p_i}{\thetahatIpert} \right) + (\phat_i -
\pmat_i)^\top (\IdMat - \discount \Phat)^{-1} (\reward - \rhat) \notag \\ &=
- \discount (\phat_i - \pmat_i)^\top (\IdMat - \discount
\Phatloo{i})^{-1} e_i \left(\inprod{\phat_i - p_i}{\thetahatIpert}
\right) + (\phat_i - \pmat_i)^\top (\IdMat - \discount
\Phatloo{i})^{-1} (\reward - \rhat) \notag \\ &\qquad - \discount (\phat_i -
\pmat_i)^\top \Mmat e_i \left(\inprod{\phat_i - p_i}{\thetahatIpert}
\right) + (\phat_i - \pmat_i)^\top \Mmat (\reward - \rhat) \notag \\ &=
\left(\inprod{\phat_i - p_i}{\thetahatIpert} \right) \cdot \frac{2
  Z^2_i - Z_i}{1 - Z_i} + T_i \cdot \frac{1 - 2 Z_i}{1 - Z_i}, \label{eq:woodbury}
\end{align}
where we have defined, for convenience, the random variables
\begin{align*}
Z_i \defn \discount (\phat_i - \pmat_i)^\top (\IdMat - \discount
\Phatloo{i})^{-1} e_i \quad \text{ and } \quad T_i \defn (\phat_i -
\pmat_i)^\top (\IdMat - \discount \Phatloo{i})^{-1} (\reward - \rhat).
\end{align*}
Since $\phat_i - \pmat_i$ is independent of the vector $(\IdMat - \discount \Phatloo{i})^{-1} (\reward - \rhat)$, applying the Hoeffding bound yields the inequality
\begin{align*}
|T_i | \leq \frac{\unicon}{1 - \discount} \| \reward - \rhat\|_\infty
\LOGDN
\end{align*}
with probability exceeding $1 - \delta / (4 \Dim)$.

On the other hand, exploiting independence between the vectors
$\phat_i - \pmat_i$ and $(\IdMat - \discount \Phatloo{i})^{-1} e_i$
and applying the Hoeffding bound, we also have
\begin{align*}
|Z_i| \leq \frac{c \discount}{1-\discount} \sqrt{ \frac{\log(8
    \Dim/\delta)}{\Nsamp}}
\end{align*}
with probability least $1 - \delta / (4 \Dim)$.  Taking $\Nsamp \geq
\unicon' \frac{\discount^2}{(1 - \discount)^2} \log(8 \Dim/\delta)$
for a sufficiently large constant $\unicon'$ ensures that $\discount |T_i | \leq
\| \reward - \rhat\|_\infty$ and $|Z_i| \leq 1/4$, so
that with probability exceeding $1 - \delta / (2\Dim)$, inequality~\eqref{eq:woodbury}
yields
\begin{align*}
\discount |\inprod{\phat_i - \pmat_i}{\thetahatIpert - \thetahat}| &\leq \unicon
\left\{ \discount | \inprod{\phat_i - p_i}{\thetahatIpert} | + \|
\reward - \rhat\|_\infty \right\} \\ 
&\leq \unicon \left\{ \discount |
\inprod{\phat_i - p_i}{\DelHatIpert} | + \discount | \inprod{\phat_i -
  p_i}{\thetastar} | + \| \reward - \rhat\|_\infty
\right\}.
\end{align*}
Finally, applying part (a) of Lemma~\ref{LemDoubleBound} completes the
proof.  \qed

\subsubsection{Proof of Lemma~\ref{LemBoundingIpert}} \label{sec:BoundingIpert}

Recall our leave-one-out matrix $\Phatloo{i}$, and the explicit
bound~\eqref{EqnAshwinOne}.  We have
\begin{align}
  \label{EqnAshwinTwo}
  \left |\inprod{\phat_i - p_i}{\thetahatIpert} \right | & \leq \left
  |\inprod{\phat_i - p_i}{\DelHatIpert} \right | + \left
  |\inprod{\phat_i - p_i}{\thetastar} \right | \; \leq \; c
  \|\DelHatIpert\|_\infty \sqrt{ \frac{\log(8 \Dim/\delta)}{\Nsamp}} +
  \left |\inprod{\phat_i - p_i}{\thetastar} \right |
\end{align}
with probability at least $1- \delta/ (4 \Dim)$.  Substituting inequality~\eqref{EqnAshwinTwo}
into the bound~\eqref{EqnExplicitIpert}, we find that
\begin{align}
  \label{EqnUsefulOne}
  \|\thetahatIpert - \thetahat\|_\infty & \leq \frac{c}{1-\discount} 
  \left\{ \discount \|\DelHatIpert\|_\infty \cdot \sqrt{ \frac{\log(8
      \Dim/\delta)}{\Nsamp}}  + \discount \left
  |\inprod{\phat_i - p_i}{\thetastar} \right | + \|
  \reward - \rhat \|_{\infty} \right\}.
\end{align}
Finally, the triangle inequality yields
\begin{align*}
  \|\DelHatIpert\|_\infty & \leq \|\DelHat\|_\infty + \|\thetahatIpert
  - \thetahat \|_\infty \\ &\leq \; \|\DelHat \|_\infty + \frac{c}{1-\discount} 
  \left\{ \discount \|\DelHatIpert\|_\infty \cdot \sqrt{ \frac{\log(8
      \Dim/\delta)}{\Nsamp}}  + \discount \left
  |\inprod{\phat_i - p_i}{\thetastar} \right | + \|
  \reward - \rhat \|_{\infty} \right\}.
\end{align*}
For $\Nsamp \geq \unicon' \discount^2 \frac{\log (8 \Dim/\delta)}{(1-
  \discount)^2}$ with $\unicon'$ sufficiently large, we have
\begin{align*}
\|\DelHatIpert\|_\infty & \leq \unicon \|\DelHat\|_\infty +
\frac{c}{1-\discount} \Big\{ \discount \left |\inprod{\phat_i - p_i}{\thetastar} \right
| + \| \rhat - \reward \|_\infty \Big\}
\end{align*}
with probability at least $1 - \frac{\delta}{4\Dim}$, which completes
the proof of Lemma~\ref{LemBoundingIpert}.
\qed

Since Corollary~\ref{cor:azar} follows from Theorem~\ref{thm:plugin},
we prove it first before moving to a proof of Theorem~\ref{thm:lb}
in Section~\ref{sec:pf-thm2}.

\subsection{Proof of Corollary~\ref{cor:azar}}

In order to prove part (a), consider
inequality~\eqref{eq:error-relation-bellman} and further use the fact
that \mbox{$\| (\IdMat -\discount \Phat)^{-1} \|_{1, \infty} \leq
  \tfrac{1}{1 - \discount}$} to obtain the element-wise bound
\begin{align*}
|\thetahat - \thetastar| & \preceq \frac{\discount}{1 - \discount} \|
(\Phat - \Pmat) \thetastar \|_{\infty} \ones + \frac{\| \rhat -
  \reward \|_{\infty}}{1 - \discount} \cdot \ones.
\end{align*}
Applying Bernstein's bound to the first term and Hoeffding's bound to
the second completes the proof.

In order to prove part (b) of the corollary, we apply Lemma 7 of Azar
et al.~\cite{AzaMunKap13}---in particular, equation (17) of that
paper. Tailored to this setting, their result leads to the point-wise
bound
\begin{align*}
\| (\IdMat - \discount \Pmat)^{-1} \sigma(\thetastar) \|_\infty \leq
\unicon \frac{\rmax}{(1 - \discount)^{3/2}}.
\end{align*}
We also have the bound 
\begin{align*}
\| \thetastar \|_{\myspan} \leq 2\| \thetastar \|_\infty = 2 \|
(\IdMat - \discount \Pmat)^{-1} \reward \|_\infty \leq \frac{2
  \rmax}{1 - \discount},
\end{align*}
so that combining the
pieces and applying Theorem~\ref{thm:plugin}(b), we obtain
\begin{align*}
\| \thetahat - \thetastar \|_\infty \leq \frac{c}{(1 - \discount)}
\left\{ \LOGDN \left( \discount \frac{\rmax}{(1 - \discount)^{1/2}} +
\| \sigmar \|_{\infty} \right) + \discount \cdot \LOGDNSQ \frac{
  \rmax}{1 - \discount} \right\}.
\end{align*}
Finally, when $\Nsamp \geq \unicon_1 \frac{\log (8\Dim / \delta)}{1 -
  \discount}$ for a sufficiently large constant $\unicon_1$, we have
\begin{align*}
\LOGDNSQ \frac{\rmax}{1 - \discount} \leq \unicon \LOGDN
\frac{\rmax}{(1 - \discount)^{1/2}},
\end{align*}
thereby establishing the claim.
\qed


\subsection{Proof of Theorem~\ref{thm:lb}} \label{sec:pf-thm2}

For all of our lower bounds, we assume that the reward distribution
takes the Gaussian form
\begin{align}
  \label{EqnGaussReward}
  \rdist(\, \cdot \, \mid j ) = \NORMAL(r_j, \sigrewbd^2)
\end{align}
for each state $j$.  Note that this reward distribution satisfies
$\|\sigmar\|_\infty = \sigrewbd$ by construction.

Let us begin with a short overview of our proof, which proceeds in two
steps. First, we suppose that the transition matrix $\Pmat$ is known
exactly, and the hardness of the estimation problem is due to noisy
observations of the reward function.  In particular, letting
$\Mspace_{\IdMat}(\rmax, \sigrewbd)$ denote the class of all MRPs with
the specific reward observation model~\eqref{EqnGaussReward}, and for
which the transition matrix is the identity matrix $\IdMat$ and the
rewards are uniformly bounded as $\|\reward \|_\infty \leq \rmax$, we
show that
\begin{align}
\label{eq:reward-lb}
\inf_{\thetahat} \sup_{\thetastar \in \Mspace_{\IdMat} (\rmax,
  \sigrewbd)} \EE \| \thetahat - \thetastar \|_{\infty} \geq c \left\{
\frac{\sigrewbd}{1 - \discount} \cdot \sqrt{\frac{\log(\Dim)}{\Nsamp}}
\land \frac{\rmax}{1 - \discount} \right\}.
\end{align}
Note that for each
pair of positive scalars $(\sigvalbd, \rmax)$ we have the inclusions
\begin{align*}
\Mspace_{\IdMat}(\rmax, \sigrewbd) \subseteq \Mspace_{\var}(\sigvalbd,
\sigrewbd) \quad \text{ and } \quad \Mspace_{\IdMat}(\rmax, \sigrewbd)
\subseteq \Mspace_{\rew}(\rmax, \sigrewbd),
\end{align*}
and so that the lower bound~\eqref{eq:reward-lb} carries over to the
classes $\Mspace_{\var}(\sigvalbd, \sigrewbd)$ and
$\Mspace_{\rew}(\rmax, \sigrewbd)$.

Next, we suppose that the population reward function $\reward$ is
known exactly ($\sigrewbd = 0$), and the hardness of the estimation
problem is only due to uncertainty in the transitions. Under this
setting, we prove the lower bounds
\begin{subequations}
\label{eq:transition-lb}
\begin{align}
\inf_{\thetahat} \; \sup_{\thetastar \in \Mspace_{\var} (\sigvalbd, 0)
} \EE \| \thetahat - \thetastar \|_{\infty} &\geq c \frac{\sigvalbd}{1
  - \discount} \cdot \sqrt{\frac{\log(\Dim/2)}{\Nsamp}}, \qquad \text{
  and } \label{eq:transition-lb-a}\\ 
  \inf_{\thetahat} \; \sup_{\thetastar \in \Mspace_{\rew}
  (\rmax, 0) } \EE \| \thetahat - \thetastar \|_{\infty} &\geq c
\frac{\rmax}{(1 - \discount)^{3/2}} \cdot
\sqrt{\frac{\log(\Dim/2)}{\Nsamp}}. \label{eq:transition-lb-b}
\end{align}
\end{subequations}
Since $\Mspace_{\var} (\sigvalbd, 0) \subset \Mspace_{\var}(\sigvalbd,
\sigrewbd)$ for any $\sigrewbd > 0$, these lower bounds also carry
over to the more general setting. The minimax lower bounds of
Theorem~\ref{thm:lb} are obtained by taking the maximum of the
bounds~\eqref{eq:reward-lb} and~\eqref{eq:transition-lb}.  Let us now
establish the two previously claimed bounds.


\subsubsection{Proof of claim~\eqref{eq:reward-lb}}

For some positive scalar $\alpha$ to be chosen shortly, consider
$\Dim$ distinct reward vectors $ \{\reward^{(1)}, \ldots,
\reward^{(\Dim)} \}$, where the vector $\reward^{(i)} \in \real^\Dim$
has entries
\begin{align*}
\reward^{(i)}_j & \defn
\begin{cases}
\alpha &\text{ if } i = j \\ 0 &\text{ otherwise,}
\end{cases}
\qquad \text{ for all } j \in [\Dim].
\end{align*}
Denote by $\MRP^{(i)}$ the MRP with reward function $\reward^{(i)}$;
and transition matrix $\IdMat$.  
Thus, the $i$-th value function is
given by the vector $(\thetastar)^{(i)} \defn \frac{1}{1 - \discount}
\reward^{(i)}$.

By construction, we have $\| (\thetastar)^{(i)} - (\thetastar)^{(j)}
\|_{\infty} = \alpha / (1 - \discount)$ for each pair of distinct
indices $(i, j)$. Furthermore, the KL divergence between Gaussians of
variance $\sigrewbd^2$ centered at $\reward^{(i)}$ and $\reward^{(j)}$
is given by
\begin{align*}
\KL \left( \NORMAL(\reward^{(i)}, \sigrewbd^2 \IdMat ) \; \| \;
\NORMAL(\reward^{(j)}, \sigrewbd^2 \IdMat ) \right) = \frac{\|
  \reward^{(i)} - \reward^{(j)} \|_2^2}{\sigrewbd^2} = \frac{2
  \alpha^2}{\sigrewbd^2}.
\end{align*}
Thus, applying the local packing version of Fano's method
(\S
15.3.3,~\cite{Wai19}), we have
\begin{align*}
\inf_{\thetahat} \; \sup_{\thetastar \in \{ \Mspace^{(i)} \}_{i \in
    [\Dim] } } \EE \| \thetahat - \thetastar \|_{\infty} \geq \unicon
\frac{\alpha}{1 - \discount} \left( 1 - \frac{2 \frac{\alpha^2}
  {\sigrewbd^{2}} \Nsamp + \log 2}{\log \Dim} \right).
\end{align*}
Setting $\alpha = \sigrewbd \sqrt{ \frac{\log \Dim}{6 \Nsamp} } \land
\rmax$ yields the claimed lower bound.

\subsubsection{Proof of claim~\eqref{eq:transition-lb}}

This lower bound is based on a modification of constructions used by
Lattimore and Hutter~\cite{LatHut14} and Azar et
al.~\cite{AzaMunKap13}.  Our proof, however, is tailored to the
generative observation model.

Our proof is structured as follows. First, we construct a family
of ``hard'' MRPs and prove a minimax lower bound as a function of 
parameters used to define this family. Constructing this
family of hard instances requires us to first define a basic building
block: a two-state MRP that was illustrated in Figure~\ref{fig:mrp}(a).
After obtaining this general lower bound, we then set the scalars that
parameterize the hard class MRP appropriately to obtain the
two claimed bounds.

We now describe the two-state MRP in more detail.  For a pair of
parameters $(p, \tau)$, each in the unit interval $[0,1]$, and a
positive scalar $\lbmaxrew$, consider the two-state Markov reward
process $\MRP_{\basic}(p, \lbmaxrew, \tau)$, with transition matrix
and reward vector given by
\begin{align*}
\Pmat_{\basic} = 
\begin{bmatrix}
p & 1 - p \\ 0 & 1
\end{bmatrix} \qquad \text{ and } \qquad 
\reward_{\basic} = 
\begin{bmatrix}
\lbmaxrew \\ \lbmaxrew \cdot \tau
\end{bmatrix},
\end{align*}
respectively.  See Figure~\ref{fig:mrp} for an illustration of this
MRP.

A straightforward calculation yields that it has value function and
corresponding standard deviation vector given by
\begin{align}
\label{eq:basic-mrp-props}
\thetastar (p, \lbmaxrew, \tau) = \lbmaxrew
\begin{bmatrix}
\frac{1 - \discount + \discount \tau (1 - p)}{(1 - \discount p) (1 -
  \discount)} \\ \frac{\tau}{1 - \discount}
\end{bmatrix}
\qquad \text{ and } \qquad
\sigma (\thetastar) = \lbmaxrew
\begin{bmatrix}
\frac{(1 - \tau) \sqrt{p (1 - p)}}{1 - \discount p} \\ 0
\end{bmatrix},
\end{align}
respectively, where we have used the shorthand $\thetastar \equiv
\thetastar(p, \lbmaxrew, \tau)$. We also have $\| \thetastar
\|_{\myspan} = \frac{\lbmaxrew (1 - \tau)}{1 - \discount p}$; the two
scalars $(\nu, \tau)$ allow us to control the quantities $\|
\sigma(\thetastar) \|_\infty$ and $\| \thetastar \|_{\myspan}$.  Index
the states of this MRP by the set $\{0, 1\}$, and consider now a
sample drawn from this MRP under the generative model.  We see a pair
of states drawn according to the respective rows of the transition
matrix $\Pmat_{\basic}$; the first state is drawn according to the
Bernoulli distribution $\BER(p)$, and the second state is
deterministic and equal to $1$. For convenience, we use $\Prob(p) =
(\BER(p), 1)$ to denote the distribution of this pair of states.

Our hard class of instances is based in part
 on the difficulty of distinguishing two such MRPs that are close in a
 specific sense. 
Let us make this intuition precise.
 For two scalar values~\mbox{$0 \leq p_2 \leq p_1 \leq
  1$}, some algebra yields the relation
\begin{align} \label{eq:linf-pair}
\| \thetastar (p_1, \lbmaxrew, \tau) - \thetastar (p_2, \lbmaxrew, \tau)
\|_{\infty} = \lbmaxrew \cdot \frac{(p_1 - p_2)(1 - \tau)}{(1 - \discount
  p_1) (1 - \discount p_2)}.
\end{align}
In the sequel, we work with the choices
\begin{align*}
p_1 = \frac{4 \discount - 1}{3 \discount} \quad \mbox{and} \quad p_2 =
p_1 - \frac{1}{8} \sqrt{ \frac{p_1(1 - p_1)}{N} \log (\Dim / 2) },
\end{align*}
which, under the assumed lower bound on the sample size $\Nsamp$, are
both scalars in the range $\big[ \tfrac{1}{2}, 1 \big)$ for all
  discount factors $\discount \in \big[ \tfrac{1}{2}, 1 \big)$.
  Moreover, it is worth noting the relations
  \begin{align}
\label{eq:discounts}
1 - p_1 & = \frac{1 - \discount}{3 \discount}, \quad &\unicon_1
\frac{1 - \discount}{3 \discount} \leq 1 - p_2 \leq \unicon_2 \frac{1
  - \discount}{3 \discount} \notag \\ 1 - \discount p_1 &= \frac{4}{3}
(1 - \discount), \quad \qquad \text{ and } \quad & c_1 (1 - \discount)
\leq 1 - \discount p_2 \leq c_2 (1 - \discount),
\end{align}
where the inequalities on the right hold provided $\Nsamp \geq \frac{c
  \discount}{1 - \discount} \log (\Dim / 2)$ for a sufficiently large
constant~$\unicon$. Here the pair of constants $(\unicon_1,
\unicon_2)$ are universal, depend only on $c$, and may change from
line to line.

We also require the following lemma, proved in
Section~\ref{sec:KL-lemma} to follow, which provides a useful bound on
the KL divergence between~$\Prob(p_1)$ and $\Prob(p_2)$.
\begin{lemma} \label{lem:KL}
For each pair $p, q \in [1/2, 1)$, we have
\begin{align*}
\KL \left( \Prob(p) \| \Prob(q) \right) \leq \frac{(p - q)^2}{ (p \lor
  q) (1 - (p \lor q))}.
\end{align*}
\end{lemma}

We are now in a position to describe the hard family of MRPs over
which we prove a general lower bound. Suppose that $\Dim$ is even for
convenience, and consider a set of $\Dim / 2$ ``master''
MRPs~\mbox{$\bar{\Mspace} \defn \{ \MRP_1, \ldots, \MRP_{\Dim / 2}
  \}$} each on $\Dim$ states\footnote{Note that this step is
  only required in order to ``tensorize'' the construction
  in order to obtain the optimal dependence on the dimension.
  If, instead of the $\ell_\infty$ error, one was interested
  in estimating the value function at a fixed state of the MRP,
  then this tensorization is no longer needed.}
   constructed as follows. Decompose each
master MRP into $\Dim / 2$ sub-MRPs of two states each; index the
$k$-th sub-MRP in the $j$-th master MRP by~$\MRP_{j, k}$. For each
pair $j, k \in [\Dim / 2]$, set
\begin{align*}
\MRP_{j , k} = 
\begin{cases}
\MRP_{\basic} (p_1, \lbmaxrew, \tau) &\text{ if } j \neq k
\\ \MRP_{\basic} (p_2, \lbmaxrew, \tau) &\text{ otherwise.}
\end{cases}
\end{align*}
Let $\thetastar_j$ denote the value function corresponding to MRP
$\MRP_j$, and let $\mathbb{P}_j^N$ denote the distribution of state
transitions observed from the MRP $\MRP_j$ under the generative model.
Also note that for each $i \in [\Dim / 2]$, we have
\begin{align}
\label{eq:sig-calc}
\| \sigma (\thetastar_i) \|_{\infty} = \lbmaxrew \frac{(1 - \tau)
  \sqrt{p_1 (1 - p_1)}}{(1 - \discount p_1)}.
\end{align}


\paragraph{Lower bounding the minimax risk over this class:}

We again use the local packing form of Fano's method (\S
15.3.3,~\cite{Wai19}) to establish a lower bound. Choose some index
$J$ uniformly at random from the set $[\Dim / 2]$, and suppose that we
draw $\Nsamp$ i.i.d. samples $Y^N \defn (Y_1, \ldots, Y_N)$ from the
MRP $\MRP_J$ under the generative model. Here each $Y_i
\in \Xspace^{\Dim}$ represents a random set of $\Dim$ states, and the
goal of the estimator is to identify the random index $J$ and,
consequently, to estimate the value function $\thetastar_J$.  Let us
now lower bound the expected error incurred in this $(\Dim / 2)$-ary
hypothesis testing problem. Fano's inequality yields the bound
\begin{align}
\label{eq:fano-hard}
\inf_{\thetahat} \sup_{\thetastar \in \bar{\Mspace} } \EE \| \thetahat
- \thetastar \|_\infty \geq \frac{1}{2} \min_{j \neq k} \; \|
\thetastar_j - \thetastar_k \|_{\infty} \left( 1 - \frac{I(J; Y^N) +
  \log 2}{\log (\Dim / 2)}\right),
\end{align}
where $I(J; Y^N)$ denotes the mutual information between $J$ and $Y^N$.

Let us now bound the two terms that appear in inequality~\eqref{eq:fano-hard}. By
equation~\eqref{eq:linf-pair}, we have
\begin{align*}
\| \thetastar_j - \thetastar_k \|_{\infty} = \lbmaxrew \cdot \frac{(p_1 -
  p_2)(1 - \tau)}{(1 - \discount p_1) (1 - \discount p_2)} \qquad
\text{ for all } 1 \leq j \neq k \leq \Dim / 2.
\end{align*}
Furthermore, since the samples $Y_1, \ldots, Y_N$ are i.i.d., the
chain rule of mutual information yields
\begin{align*}
\frac{1}{\Nsamp} I(J; Y^N) = I(J; Y_1) & \leq \max_{j \neq k} \; \KL
(\mathbb{P}_j \| \mathbb{P}_k ) \\ 
&\stackrel{\1}{=} \KL( \Prob(p_1) \| \Prob(p_2)) +
\KL( \Prob(p_2) \| \Prob(p_1))
\\ &\stackrel{\2}{\leq} 2 \frac{(p_1 - p_2)^2}{p_1 (1 - p_1)},
\end{align*}
where step $\1$ is a consequence of the construction, which ensures that the distributions
$\mathbb{P}_j$ and $\mathbb{P}_k$ coincide on all but the $j$-th and $k$-th sub-MRPs. 
On the other hand, step $\2$ follows from Lemma~\ref{lem:KL}, and the fact that $p_2 \leq p_1$.

Putting together the pieces, we now have
\begin{align*}
\inf_{\thetahat} \sup_{\thetastar \in \bar{\Mspace} } \EE \| \thetahat
- \thetastar \|_\infty \geq \frac{\lbmaxrew}{2} \cdot \frac{(p_1 - p_2)(1 -
  \tau)}{(1 - \discount p_1) (1 - \discount p_2)} \left( 1 -
\frac{2\Nsamp \frac{(p_1 - p_2)^2}{p_1 (1 - p_1)} + \log 2}{\log
  (\Dim/2)} \right).
\end{align*}
Recall the choice $p_1 - p_2 = \frac{1}{8} \sqrt{ \frac{p_1 (1 -
    p_1)}{N} \log (\Dim/2)}$. For $\Dim \geq 8$, this ensures, for a
small enough positive constant $c$, the bound
\begin{align} \label{eq:general-lb}
\inf_{\thetahat} \sup_{\thetastar \in \bar{\Mspace} } \EE \| \thetahat - \thetastar \|_\infty \geq c \lbmaxrew \frac{(1 - \tau) \sqrt{p_1 (1 - p_1)}}{(1 - \discount p_1)} \cdot \sqrt{ \frac{\log (\Dim/2)}{N}} \frac{1}{1 - \discount p_2}.
\end{align}
With the relation~\eqref{eq:general-lb} at hand, we now turn to proving the two sub-claims
in equation~\eqref{eq:transition-lb}.

\paragraph{Proof of claim~\eqref{eq:transition-lb-a}:}
Recall equation~\eqref{eq:sig-calc}; for $i \in [\Dim / 2]$, we have
\begin{align*}
\| \sigma (\thetastar_i) \|_{\infty} = \lbmaxrew \frac{(1 - \tau)
  \sqrt{p_1 (1 - p_1)}}{(1 - \discount p_1)}  \quad \text{ and } \quad \| \thetastar_i \|_{\myspan} = 
  \lbmaxrew \frac{(1 - \tau)}{(1 - \discount p_1)}.
\end{align*}
Now for every pair of scalars $(\sigvalbd, \valbd)$ satisfying
$\sigvalbd = \valbd \sqrt{1 - \discount}$, set $\tau = 1/2$ and
$\lbmaxrew = 2 \valbd (1 - \discount p_1)$.  With this choice of
parameters, we have the inclusion $\Mspace_{\var}(\sigvalbd, 0)
\cap \Mspace_{\val}(\valbd, 0) \subseteq \bar{\Mspace}$, and
evaluating the bound~\eqref{eq:general-lb} yields
\begin{align*}
\inf_{\thetahat} \sup_{\thetastar \in \MspaceVar(\sigvalbd, 0) \cap \Mspace_{\val}(\valbd, 0)} \EE
\| \thetahat - \thetastar \|_\infty &\geq c \sigvalbd \sqrt{
  \frac{\log (\Dim/2)}{N}} \frac{1}{1 - \discount p_2}
\\ &\stackrel{\2}{=} c \sigvalbd \sqrt{ \frac{\log (\Dim/2)}{N}}
\frac{1}{1 - \discount},
\end{align*}
where in step $\2$, we have used inequality~\eqref{eq:discounts}. 
The same lower bound clearly also extends to the set 
$\MspaceVar(\sigvalbd, 0) \cap \Mspace_{\val}(\valbd, 0)$
for $\valbd \geq \sigvalbd (1 - \discount)^{-1/2}$;
this establishes part (a) of the theorem.

\paragraph{Proof of claim~\eqref{eq:transition-lb-b}:}
Given a value $\rmax$, set $\tau = 0$ and $\lbmaxrew = \rmax$ and note that the
rewards of all the MRPs in the set $\bar{\Mspace}$ satisfy $\| \reward \|_\infty \leq
\lbmaxrew$. Hence, we have $\Mspace_{\rew} (\rmax, 0) \subseteq \bar{\Mspace}$
for this choice of parameters. Using inequality~\eqref{eq:general-lb} and recalling 
the bounds~\eqref{eq:discounts} once again, we
have
\begin{align*}
\inf_{\thetahat} \sup_{\thetastar \in \MspaceRew(\rmax, 0) } \EE \|
\thetahat - \thetastar \|_\infty &\geq c \rmax \frac{\sqrt{p_1 (1 -
    p_1)}}{(1 - \discount p_1)} \cdot \sqrt{ \frac{\log (\Dim/2)}{N}}
\frac{1}{1 - \discount p_2} \\ &\geq c \frac{\rmax}{(1 -
  \discount)^{3/2}} \sqrt{ \frac{\log (\Dim/2)}{N}}.
\end{align*}

\subsubsection{Proof of Lemma~\ref{lem:KL}}
\label{sec:KL-lemma}

By construction, the second state of the Markov chain is absorbing, so
it suffices to consider the KL divergence between the first components
of the distributions $\Prob(p)$ and $\Prob(q)$. These are Bernoulli
random variables $\BER(p)$ and $\BER(q)$, and the following
calculation bounds their KL divergence:
\begin{align*}
\KL \left( \BER(p) \| \BER(q) \right) &= p \log \frac{p}{q} + (1 - p)
\log \frac{1 - p}{1 - q} \\ &\stackrel{\2}{\leq} p \cdot \frac{p -
  q}{q} + (1 - p) \cdot \frac{q - p}{1 - q} \\ &= \frac{(p - q)^2}{q
  (1 - q)},
\end{align*}
where step $\2$ uses the inequality $\log (1 + x) \leq x$, which is
valid for all $x > -1$.  A similar inequality holds with the roles of
$p$ and $q$ reversed, and the denominator of the expression is lower
for the larger value $p \lor q$. This completes the proof.


\subsection{Proof of Theorem~\ref{thm:rob}}

Note that the median-of-means operator is applied elementwise; denote
the $i$-th such operator by~$\MoM_i$. Let $\MoM - \Pmat$ denote the
elementwise difference of operator $\MoM$ and the linear operator
$\Pmat$; its $i$-th component is given by the operator $\MoM_i(\cdot)
- \inprod{\pmat_i}{\cdot}$.

We require two technical lemmas in the proof.  The power of the
median-of-means device is clarified by the first lemma, which is an
adaptation of classical results (see,
e.g.,~\cite{NemYu83,jerrum1986random}).
\begin{lemma}
\label{lem:MoM}
Suppose that $K = 8 \log (4 \Dim / \delta)$ and $m = \lfloor \Nsamp / K \rfloor$. Then
there is a universal constant $\unicon$ such that for each index $i
\in [\Dim]$ and each fixed vector $\theta \in \real^{\Dim}$, we have
\begin{align*}
\Pr \left\{ |(\MoM_i - \pmat_i) (\theta)| \geq c \; \sigma_i (\theta)
\sqrt{ \frac{\log(8 \Dim / \delta)}{\Nsamp} } \right\} \leq
\frac{\delta}{4 \Dim}.
\end{align*}
\end{lemma}
Comparing this lemma to the Bernstein bound
(cf. equation~\eqref{EqnBasicBern}), we see that we no longer pay in
the span semi-norm $\| \thetastar \|_{\myspan}$, and this is what
enables us to establish the solely variance-dependent
bound~\eqref{eq:robust-result}.

We also require the following lemma that guarantees that the
median-of-means Bellman operator is contractive.
\begin{lemma}
\label{lem:op-contracts}
The median-of-means operator is $1$-Lipschitz in the
$\ell_\infty$-norm, and satisfies
\begin{align*}
| \MoM(\theta_1) - \MoM(\theta_2) | \leq \| \theta_1 - \theta_2
\|_{\infty} \qquad \text{ for all vectors } \theta_1, \theta_2 \in
\real^{\Dim}.
\end{align*}
Consequently, the empirical operator $\Belemp^{\robust}_N$ is
$\discount$-contractive in $\ell_\infty$-norm and satisfies
\begin{align*}
|\Belemp^{\robust}_N (\theta_1) - \Belemp^{\robust}_N (\theta_2) |
\leq \discount \| \theta_1 - \theta_2 \|_{\infty} \qquad \text{ for
  all pairs of value functions } (\theta_1, \theta_2).
\end{align*}
\end{lemma}
\noindent See Section~\ref{SecOpContracts} for the proof of
Lemma~\ref{lem:op-contracts}. \\

\noindent We are now in a position to establish the theorem, where we
now use the shorthand $\thetahat \equiv \ValueFunchatrob$ for
convenience. Note that the vectors $\thetastar$ and $\thetahat$
satisfy the fixed point relations
\begin{align*}
\thetastar = r + \discount \Pmat \thetastar, \quad \mbox{and} \quad
\thetahat = \rhat + \discount \MoM( \thetahat),
\end{align*}
respectively.  Taking differences, the error vector $\DelHat =
\thetahat - \thetastar$ satisfies the relation
\begin{align*}
\thetahat - \thetastar &= \discount (\MoM(\DelHat + \thetastar) -
\Pmat \thetastar) + \rhat - \reward \\ &= \discount (\MoM(\DelHat +
\thetastar) - \MoM(\thetastar) ) + \discount (\MoM - \Pmat)
(\thetastar) + (\rhat - \reward).
\end{align*}
Taking $\ell_\infty$-norms on both sides and using the triangle
inequality, we have
\begin{align*}
\| \DelHat\|_{\infty} &\leq \discount \| \MoM(\thetastar + \DelHat) -
\MoM(\thetastar) \|_{\infty} + \discount |(\MoM - \Pmat) (\thetastar)|
+ \| \rhat - \reward \|_\infty \\
& \stackrel{\1}{\leq} \discount \| \DelHat\|_{\infty} + \discount
|(\MoM - \Pmat) (\thetastar)| + \| \rhat - \reward \|_\infty,
\end{align*}
where step $\1$ is a result of Lemma~\ref{lem:op-contracts}. Finally,
applying Lemma~\ref{lem:MoM} in conjunction with the Hoeffding
inequality and a union bound over all $\Dim$ indices completes the
proof.

\subsubsection{Proof of Lemma~\ref{lem:op-contracts}}
\label{SecOpContracts}
The second claim follows directly from the first by noting that
\begin{align*}
|\Belemp^{\robust}_N (\theta_1) - \Belemp^{\robust}_N (\theta_2) | = \discount | \MoM(\theta_1) - \MoM(\theta_2) |.
\end{align*}
In order to prove the first claim, recall that for each $\theta \in
\real^{\Dim}$, we have $\MoM(\theta) = \med (\muhat_1 (\theta),
\ldots, \muhat_K (\theta))$, where the median---defined as the
$\lfloor K /2 \rfloor$-th order statistic---is taken entry-wise.  By
definition, for each $i \in [K]$, we have
\begin{align*}
\| \muhat_i (\theta_1) - \muhat_i (\theta_2) \|_\infty &= \left\|
\left( \frac{1}{m} \sum_{k \in \Dspace_i} \Zmatsub{k} \right)
(\theta_1 - \theta_2) \right\|_{\infty} \\ &\leq \left\| \frac{1}{m}
\sum_{k \in \Dspace_i} \Zmatsub{k} \right\|_{1, \infty} \| \theta_1 -
\theta_2 \|_{\infty} \\ &\stackrel{\1}{=} \| \theta_1 - \theta_2
\|_{\infty},
\end{align*}
where step~$\1$ is a
result of the fact that $\frac{1}{m} \sum_{k \in \Dspace_i}
\Zmatsub{k}$ is a row stochastic matrix with non-negative entries.
Finally, we have the entry-wise bound
\begin{align*}
|\MoM(\theta_1) - \MoM(\theta_2)| &= |\med (\muhat_1 (\theta_1),
\ldots, \muhat_K (\theta_1)) - \med( \muhat_1 (\theta_2), \ldots,
\muhat_K(\theta_2))| \\ &\stackrel{\2}{\preceq} \| \theta_1 - \theta_2
\|_{\infty} \cdot \ones,
\end{align*}
where step $\2$ follows from our definition of the median as the
$\lfloor K/ 2 \rfloor$-th order statistic, and
Lemma~\ref{lem:order-stat} to follow. This completes the proof of
Lemma~\ref{lem:op-contracts}. \qed

\begin{lemma}
\label{lem:order-stat}
For each pair of vectors $(u, v)$ of dimension $\Dim$ and each index
$i \in [\Dim]$, we have
\begin{align*}
| u_{(i)} - v_{(i)} | \leq \| u - v \|_{\infty}.
\end{align*}
\end{lemma}
\begin{proof}
  Assume without loss of generality that the entries of $u$ are sorted
  in increasing order (so that $u_1 \leq u_2 \leq \ldots \leq
  u_\Dim$), and let $w$ denote a vector containing the entries of $v$
  sorted in increasing order. We then have
\begin{align*}
| u_{(i)} - v_{(i)} | = | u_i - w_i | \leq \| u - w \|_{\infty}
\stackrel{\1}{\leq} \| u - v \|_{\infty},
\end{align*}
where step $\1$ follows from the rearrangement inequality applied to
the $\ell_\infty$-norm~\cite{vince90a}.
\end{proof}


\section{Discussion}
\label{SecDiscussion}

Our work investigates the local minimax complexity of value function
estimation in Markov reward processes. Our upper bounds are
instance-dependent, and we also provide minimax lower bounds that hold
over natural subsets of the parameter space. The plug-in approach is
shown to be optimal over the class of MRPs with bounded rewards, and a
variant based on the median-of-means device achieves optimality over
the class of MRPs having value functions with bounded variance.

Our results also leave a few interesting technical questions unresolved. 
Let us start with two inter-related questions: Is
Corollary~\ref{cor:azar}(a) sharp, say up to a logarithmic factor
in the dimension?  Is
the median-of-means approach minimax-optimal over the class of MRPs
having bounded rewards? We conjecture that both of these questions
can be answered in the affirmative, but there are technical challenges to overcome.
For instance, while the median-of-means device is crucial to removing the span-norm
dependence in Corollary~\ref{cor:azar}(a), it leads to a non-linear update rule
that needs to be much more carefully handled in order to ensure minimax optimality. 

A second set of technical questions concerns the definition of ``locality" in our bounds. 
Are our results also sharp under alternative local minimax parameterizations (say
 in terms of the functional $\| (\IdMat - \discount \Pmat)^{-1} \sigma(\ValueFunc^*) \|_{\infty}$)?
Is there a more fine-grained lower bound
analysis that shows the (sub)-optimality of these approaches, and are
there better adaptive procedures for this problem? The literature on
estimating functionals of discrete
distributions~\cite{jiao2015minimax} shows that additional refinements
over the plug-in approach are usually beneficial; is that also the
case here?  There is also the related question of whether a minimax
lower bound can be proved over a local neighborhood of every point
$\thetastar$. We remark that guarantees of this flavor exist in a
variety of related
problems in both the asymptotic and non-asymptotic
settings~\cite{van2000asymptotic,cai2004adaptation,zhu2016local}. 
Indeed, in a follow-up paper~\cite{khamaru2020temporal} with a superset of the current
authors, we have shown such a local lower bound for this problem, 
which is achieved via stochastic approximation coupled with a variance reduction device.

%
In a complementary direction, another interesting question is to
ask how function approximation affects these bounds.
Our techniques should be useful in answering some of these questions,
and also more broadly in proving analogous guarantees in the more
challenging policy optimization setting. 

Finally, there is the question of removing our assumption on the
generative model: How does the plug-in estimator behave when it is
computed on a sampled \emph{trajectory} of the system? A classical
solution is the blocking method of simulating the generative model
from such samples~\cite{yu1994rates}: given a sampled trajectory, chop
it into pieces of length (roughly) equal to the mixing time of the
Markov chain, and to treat the respective first sample from each of
these pieces as (approximately) independent. But clearly, this
approach is somewhat wasteful, and there have been recent refinements
in related problems when the mixing time can become arbitrarily
large~\cite{simchowitz2018learning}.  It would be interesting to
explore these approaches and derive instance-dependent guarantees in
the $L^2_{\mu}$-norm, where $\mu$ is the stationary distribution of
the Markov chain.

\subsection*{Acknowledgements}
AP was supported in part by a research fellowship from the Simons Institute for the Theory
of Computing. MJW and AP were partially supported by National Science Foundation Grant
DMS-1612948 and Office of Naval Research Grant ONR-N00014-18-1-2640.

\appendix
\section{Dependence of plug-in error on span semi-norm}
\label{AppPluginLB}

In this section, we state and prove a proposition that provides a
family of MRPs in which the $\ell_\infty$-error of the plug-in
estimator can be completely characterized by the span semi-norm of the
optimal value function.

\begin{proposition}
\label{prop:plugin-lb}
Suppose that the rewards are observed noiselessly, with $\sigmar = 0$.
There is a pair of universal positive constants $(\unicon_1,
\unicon_2)$ such that for any triple of positive scalars $(\valbd,
\Nsamp, \Dim)$, there is a $\Dim$-state MRP for which
\begin{align}
\label{eq:prop-var-a}
\| \thetastar \|_{\infty} = \valbd \quad \text{ and } \quad \frac{\|
  \sigma(\ValueFunc^*) \|_{\infty}}{\Nsamp} \leq \frac{3}{\sqrt{\Dim}}
\cdot \frac{\valbd}{\Nsamp},
\end{align}
and for which the error of the plug-in estimator satisfies
\begin{align}
\label{eq:prop-var-b}
\unicon_1 \discount \frac{\valbd}{\Nsamp} \quad \stackrel{(a)}{\leq}
\; \EE \left[ \|\ValueFunchat - \ValueFunc^* \|_\infty \right] \;
\stackrel{(b)}{\leq} \; \unicon_2 \discount \frac{\valbd \log (1 +
  \Dim / 3)}{\Nsamp} \left\{ (\log \log (1 + \Dim / 3))^{-1} \land 1
\right\}.
\end{align}
\end{proposition}
A few comments are in order. First, note that
equation~\eqref{eq:prop-var-a} guarantees that we have
\begin{align*}
\frac{1}{\sqrt{\Nsamp}} \cdot \| (\IdMat - \discount \Pmat)^{-1}
\sigma(\ValueFunc^*) \|_{\infty} \lesssim \frac{1}{\sqrt{D}}
\frac{1}{(1 - \discount) \Nsamp} \| \ValueFunc^* \|_{\myspan}
\end{align*}
for large values of the dimension $\Dim$, so that the first term in
the guarantee~\eqref{eq:thm1-pop} is dominated by the second. In
particular, suppose that $\Dim \gg \frac{1}{(1 - \discount)^2}$; then
we have
\begin{align*}
\frac{1}{\sqrt{\Nsamp}} \cdot \| (\IdMat - \discount \Pmat)^{-1} \sigma(\ValueFunc^*) \|_{\infty} \ll \frac{1}{\Nsamp} \| \ValueFunc^* \|_{\myspan}.
\end{align*}
In other words, if our analysis was loose in that the error of the
plug-in estimator depended only on the functional
$\frac{1}{\sqrt{\Nsamp}} \cdot \| (\IdMat - \discount \Pmat)^{-1}
\sigma(\ValueFunc^*) \|_{\infty}$, then it would be impossible to
prove a lower bound that involves the quantity $\frac{\| \thetastar
  \|_{\myspan}}{\Nsamp}$.  On the other hand,
equation~\eqref{eq:prop-var-b} shows that this such a lower bound can
indeed be proved: the plug-in error is characterized precisely by the
quantity $\discount \frac{\| \thetastar \|_{\myspan}}{\Nsamp}$ up to a
logarithmic factor in the dimension $\Dim$.

Second, note that while equation~\eqref{eq:prop-var-b} shows that the
plug-in error must have some span semi-norm dependence, it falls short
of showing the stronger lower bound
\begin{align}
\label{eq:conj}
\unicon_1 \frac{\valbd}{(1 - \discount)\Nsamp} \leq \EE \left[
  \|\ValueFunchat - \ValueFunc^* \|_\infty \right],
\end{align}
which would show, for instance, that Corollary~\ref{cor:azar}(a)
is sharp up to a logarithmic factor. We conjecture that
there is an MRP for which the bound~\eqref{eq:conj} holds.

Finally, it is worth commenting on the logarithmic factor that
appears in the upper bound of equation~\eqref{eq:prop-var-b}.
Note that for sufficiently large $\Dim$, the logarithmic factor
is proportional to $\log D / \log \log D$. This is consequence
of applying Bennett's inequality instead of Bernstein's inequality,
and we conjecture that the same factor ought to replace the
factor $\log \Dim$ factor multiplying the span semi-norm
in the upper bound~\eqref{eq:thm1-pop}.


\subsection{Proof of Proposition~\ref{prop:plugin-lb}}

In order to prove Proposition~\ref{prop:plugin-lb}, it suffices to
construct an MRP satisfying condition~\eqref{eq:prop-var-a} and
compute its plug-in estimator in closed form.  With this goal in mind,
suppose that for simplicity that $\Dim$ is divisible by three, and
consider $\Dim/3$ copies of the $3$-state MRP from
Figure~\ref{fig:mrp2}(a). By construction, we have $\|
\sigma(\ValueFunc^*) \|_{\infty} = \secondmrpbd \sqrt{q (1 - q)}$, and
$\| \ValueFunc^* \|_{\myspan} = \frac{\secondmrpbd}{(1 -
  \discount)}$. Setting $q = \frac{10}{\Nsamp \Dim}$, we see that
condition~\eqref{eq:prop-var-a} is immediately satisfied with $\valbd
= \frac{\secondmrpbd}{(1 - \discount)}$.

It remains to verify the claim~\eqref{eq:prop-var-b}.  Note that the
plug-in estimator for this MRP can be computed in closed form. In
particular, it is straightforward to verify that for each state $i$
having reward $\mu/2$, we have
\begin{align}
\label{eq:dist}  
\frac{\Nsamp (1 - \discount)}{\discount \secondmrpbd} \left(
\ValueFunchat(i) - \ValueFunc^*_i\right) \stackrel{d}{=} \BIN (\Nsamp,
q) - \Nsamp q,
\end{align}
where we have used the notation $\ValueFunchat(i)$ to denote the
$i$-th entry of the vector $\ValueFunchat$. Furthermore, these
$\Dim/3$ random variables are independent. Thus, the (scaled)
$\ell_\infty$-error of the plug-in estimator is equal to the maximum
absolute deviation in a collection of independent binomial random
variables.

\paragraph{Proof of inequality~\eqref{eq:prop-var-b}, part (a):}

The following technical lemma provides a lower bound on the deviation
of binomials, and its proof is postponed to the end of this section.
\begin{lemma} \label{lem:bin-prob}
Let $X_1, \ldots, X_k$ denote independent random variables with
distribution $\BIN \left(n, \frac{1}{3kn} \right)$. Let $Y_j = X_j -
\EE[X_j]$ for each $1 \leq j \leq k$. Then, we have
\begin{align*}
\EE \left[ \max_{1 \leq j \leq k} | Y_j| \right] \geq \frac{4}{9}.
\end{align*}
\end{lemma}
Applying Lemma~\ref{lem:bin-prob} with $k = \Dim/3$ in conjunction
with the characterization~\eqref{eq:dist}, and substituting our
choices of the pair $(\secondmrpbd, q)$ yields
\begin{align*}
\EE \left[ \| \ValueFunchat - \ValueFunc^* \|_{\infty} \right] \geq
\frac{4}{9} \cdot \frac{\valbd \discount}{\Nsamp}.
\end{align*}

\paragraph{Proof of inequality~\eqref{eq:prop-var-b}, part (b):}

Corollary 3.1(ii) and Lemma 3.3 of Wellner~\cite{wellner2017bennett}
yield, to the best of our knowledge, the sharpest available upper
bound on the maximum absolute deviation of $\BIN(n,q)$ random
variables in the regime $nq(1 - q) \ll 1$:
\begin{align}
\label{eq:bennett}
\EE \left[ \max_{1 \leq j \leq k} | Y_j| \right] \leq \sqrt{12} \cdot
\frac{\log (1 + k)}{\log \log (1 + k)} \quad \text{ if } \quad \log (1
+ k) \geq 5.
\end{align}
Combining this bound with the Bernstein bound when $k$ is small, and
substituting the various quantities completes the proof.

\paragraph{Proof of Lemma~\ref{lem:bin-prob}:} Employing the shorthand $q = \frac{1}{3kn}$, we have
\begin{align*}
\EE \left[ \max_{1 \leq j \leq k} | Y_j| \right] &\geq \left( 1 -
nq\right) \cdot \Pr \left\{ \max_{1 \leq j \leq k} X_j \geq 1 \right\}
\\ &= \left( 1 - nq\right) \cdot (1 - (1 - q)^{nk})\\ &\geq
\frac{2}{3} \cdot \left( 1 - \sqrt[3]{\left( 1 - \frac{1}{3nk}
  \right)^{3nk}} \right) \\ &\geq \frac{2}{3} \cdot \left( 1 -
e^{-1/3}\right) \geq \frac{4}{9}.
\end{align*}


\section{Calculations for the ``hard'' sub-class}
\label{AppAux}

Recall from equation~\eqref{eq:basic-mrp-props} our previous
calculation of the value function and standard deviation, from which
we have
\begin{align*} 
\| \sigma (\thetastar) \|_{\infty} = \lbmaxrew (1 - \tau)
\frac{\sqrt{p(1 - p)}}{1 - \discount p}, \qquad \| (\IdMat - \discount
\Pmat)^{-1} \sigma (\thetastar) \|_{\infty} = \lbmaxrew (1 - \tau)
\frac{\sqrt{p(1 - p)}}{(1 - \discount p)^2},
\end{align*}
and $\| \thetastar \|_{\myspan} = \lbmaxrew (1 - \tau) \frac{1}{1 -
  \discount p}$.  Substituting in our choices $\lbmaxrew = 1$, $p =
\frac{4\discount -1}{3 \discount}$, and $\tau = 1 -
(1-\discount)^\alpha$ and simplifying by employing
inequality~\eqref{eq:discounts}, we have
\begin{align*}
 \| \sigma(\thetastar) \|_{\infty} \sim \left( \frac{1}{1 - \discount}
 \right)^{0.5 - \alpha}, \; \| (\IdMat - \discount \Pmat)^{-1}
 \sigma(\thetastar) \|_\infty \sim \left( \frac{1}{1 - \discount}
 \right)^{1.5 - \alpha}, \; \text{ and } \; \| \thetastar \|_{\myspan}
 \sim \left( \frac{1}{1 - \discount} \right)^{1 - \alpha},
 \end{align*}
for each discount factor $\discount \geq \tfrac{1}{2}$.  Here, the
$\sim$ notation indicates that the LHS can be sandwiched between two
terms that are proportional to the RHS such that the factors of
proportionality are strictly positive and $\discount$-independent.

For the plug-in estimator, its performance will be determined by the
maximum of the two terms
\begin{align*}
  \frac{\| (\IdMat - \discount \Pmat)^{-1} \sigma(\thetastar) \|_\infty}{\sqrt{\Nsamp}}
   \sim \frac{1}{\sqrt{\Nsamp}} \left( \frac{1}{1 - \discount} \right)^{1.5 - \alpha} \quad \mbox{and} \quad
\frac{\spannorm{\thetastar}}{(1-\discount) \Nsamp} \sim
  \frac{1}{\Nsamp} \left( \frac{1}{1 - \discount} \right)^{2 -
    \alpha}.
\end{align*}
In the regime $\Nsamp \succsim \frac{1}{1-\discount}$, the first term
will be dominant.

\bibliographystyle{alpha}
\bibliography{MRPs}

\end{document}

%% file: commands-cleaned.txt
\usepackage{comment}
\usepackage{epsf}
\usepackage{graphicx}

\usepackage{subfigure}

\usepackage{color}

\usepackage{mathtools}
\usepackage{amsfonts}
\usepackage{amsthm}
\usepackage{amsmath, bm}
\usepackage{amssymb}

\usepackage{pgfplots}
\pgfplotsset{width=5.35cm,compat=1.9}

\usepackage{textcomp}
\usepackage{xcolor}

\usepackage[ruled, vlined, lined, commentsnumbered]{algorithm2e}

\newtheorem{theorem}{Theorem}

\newtheorem{proposition}{Proposition}
\newtheorem{lemma}{Lemma}
\newtheorem{corollary}{Corollary}

\long\def\comment#1{}

\newcommand{\red}[1]{\textcolor{red}{#1}}

\newcommand{\defn}{:\,=}

\newcommand{\numobs}{\ensuremath{n}}

\newcommand{\ones}{\ensuremath{{\bf 1}}}

\newcommand{\inprod}[2]{\ensuremath{\langle #1 , \, #2 \rangle}}

\newcommand{\KL}{\ensuremath{D_{\mathsf{KL}}}}

\newcommand{\EE}{\ensuremath{{\mathbb{E}}}}
\newcommand{\Prob}{\ensuremath{{\mathbb{P}}}}

\newcommand{\ind}[1]{\ensuremath{{\mathbf{1}\left\{ #1 \right\}}}}

\newcommand{\1}{\ensuremath{{\sf (i)}}}
\newcommand{\2}{\ensuremath{{\sf (ii)}}}

\newcommand{\NORMAL}{\ensuremath{\mathcal{N}}}
\newcommand{\BER}{\ensuremath{\mbox{\sf Ber}}}
\newcommand{\BIN}{\ensuremath{\mbox{\sf Bin}}}

\newcommand{\Dspace}{\ensuremath{\mathcal{D}}}
\newcommand{\Xspace}{\ensuremath{\mathcal{X}}}

\newcommand{\Espace}{\ensuremath{\mathcal{E}}}

\newcommand{\Mspace}{\ensuremath{\mathfrak{M}}}

\newcommand{\thetastar}{\ensuremath{\theta^*}}
\newcommand{\thetahat}{\ensuremath{\widehat{\theta}}}

\newcommand{\real}{\ensuremath{\mathbb{R}}}

\newcommand{\st}{\ensuremath{\; \Big | \;}}

\newcommand{\ValueFunc}{\ensuremath{\theta}}
\newcommand{\ValueFunchat}{\ensuremath{\widehat{\theta}_{\mathsf{plug}}}}
\newcommand{\ValueFunchatrob}{\ensuremath{\widehat{\theta}_{\mathsf{MoM}}}}

\newcommand{\Belemp}{\ensuremath{\widehat{\mathcal{T}}}}

\newcommand{\MoM}{\ensuremath{\widehat{\mathcal{M}}}}
\newcommand{\muhat}{\ensuremath{\widehat{\mu}}}

\newcommand{\robust}{\ensuremath{\mathsf{MoM}}}

\newcommand{\rdist}{\ensuremath{\mathcal{D}_r}}
\newcommand{\sigmar}{\ensuremath{\rho(\rstar)}}

\newcommand{\Nsamp}{\ensuremath{N}}
\newcommand{\Dim}{\ensuremath{D}}

\newcommand{\phat}{\ensuremath{\widehat{p}}}

\newcommand{\DelHatIpert}{\ensuremath{\DelHat^{(i)}}}
\newcommand{\thetahatIpert}{\ensuremath{\thetahat^{(i)}}}
\newcommand{\pmat}{\ensuremath{p}}
\newcommand{\Phatloo}[1]{\Phat^{(#1)}}

\newcommand{\basic}{\mathbf{0}}
\newcommand{\secondbasic}{\mathbf{1}}
\newcommand{\secondmrpbd}{\mu}

\newcommand{\val}{\mathsf{vfun}}
\newcommand{\rew}{\mathsf{rew}}

\newcommand{\sigvalbd}{\ensuremath{\vartheta}}
\newcommand{\sigrewbd}{\ensuremath{\varrho}}
\newcommand{\valbd}{\ensuremath{\zeta}}

\newcommand{\lbmaxrew}{\ensuremath{\nu}}

%% file: DinosaurMacros
\newcommand{\LOGDN}{\ensuremath{\sqrt{\LOGDNSQ}}}
\newcommand{\LOGDNSQ}{\ensuremath{\frac{\log(8 \Dim/\delta)}{\Nsamp}}}

\newcommand{\sighat}{\ensuremath{\widehat{\sigma}}}
\newcommand{\IdMat}{\ensuremath{\mathbf{I}}}
\newcommand{\Phat}{\ensuremath{\widehat{\Pmat}}}

\newcommand{\Id}{\ensuremath{\mathbf{I}}}

\newcommand{\onevec}{\ensuremath{\mathbf{1}}}

\newcommand{\state}{\ensuremath{x}}

\newcommand{\reward}{\ensuremath{r}}
\newcommand{\rhat}{\ensuremath{\widehat{r}}}
\newcommand{\Reward}{\ensuremath{R}}

\newcommand{\discount}{\ensuremath{\gamma}}

\newcommand{\StateSpace}{\ensuremath{\mathcal{X}}}

\newcommand{\myspan}{\ensuremath{\operatorname{span}}}

\newcommand{\bcar}{\begin{itemize}}
\newcommand{\ecar}{\end{itemize}}

\newcommand{\unicon}{\ensuremath{c}}

\newcommand{\DelHat}{\ensuremath{\widehat{\Delta}}}

\newcommand{\rmax}{\ensuremath{r_{\max}}}

\newcommand{\Ehat}{\ensuremath{\widehat{\mathbb{E}}}}

\definecolor{MyGray}{rgb}{0.9,0.9,0.9}

\makeatletter\newenvironment{graybox}{ 
\begin{lrbox}{\@tempboxa}
\begin{minipage}{0.985\columnwidth}}{\end{minipage}
\end{lrbox}%
\colorbox{MyGray}{\usebox{\@tempboxa}} }
\makeatother



\newcommand{\Zback}[1]{\ensuremath{Z^{\backslash i}}}

\newcommand{\xsamstack}[1]{\ensuremath{x_1^\numobs}}
\newcommand{\Xsamstack}[1]{\ensuremath{X_1^\numobs}}

\newcommand{\widgraph}[2]{\includegraphics[keepaspectratio,width=#1]{#2}}

\newcommand{\Term}{\ensuremath{T}}

\newcommand{\fancysoln}[1]{
\ifthenelse{\equal{\doctype}{WITHSOLS}}
{
\begin{soln}
#1
\end{soln}
}
{
}
}

\newcommand{\med}{\ensuremath{\operatorname{med}}}

\newcommand{\pvec}[1]{\ensuremath{p_{#1}}}

\long\def\comment#1{}

\makeatletter
\def\@cite#1#2{[\if@tempswa #2 \fi #1]}
\makeatother

\newcommand{\var}{\ensuremath{\operatorname{var}}}

\newcommand{\mprob}{\ensuremath{\mathbb{P}}}

\newcommand{\mymathbf}[1]{\ensuremath{\mathbf{#1}}}

\newcommand{\Zmat}{\ensuremath{\mymathbf{Z}}}
\newcommand{\Zmatsub}[1]{\ensuremath{\mymathbf{Z}}_{#1}}

\newcommand{\Mmat}{\ensuremath{\mymathbf{M}}}

\newcommand{\Amat}{\ensuremath{\mymathbf{A}}}

\newcommand{\Pmat}{\ensuremath{\mymathbf{P}}}

\newcommand{\betastar}{\ensuremath{\beta^*}}

\newlength{\widebarargwidth}
\newlength{\widebarargheight}
\newlength{\widebarargdepth}

\newcommand{\Exs}{\ensuremath{\mathbb{E}}}